\documentclass[11pt]{article}
\usepackage{graphicx}
\newcount\Comments
\usepackage[colorlinks=true, allcolors=blue]{hyperref}
\usepackage{bbm}
\usepackage{amsmath}
\usepackage{amssymb}
\usepackage{algorithm}
\usepackage[framemethod=TikZ]{mdframed}
\usepackage{mathtools}
\usepackage{amsthm}
\usepackage{letltxmacro}
\usepackage{authblk}
\usepackage{mathrsfs} 
\usepackage{float}

\usepackage{booktabs}
\usepackage{makecell}

\usepackage{abstract}

\PassOptionsToPackage{algo2e,ruled, linesnumbered}{algorithm2e}
\RequirePackage{algorithm2e}

\LetLtxMacro\amsproof\proof
\LetLtxMacro\amsendproof\endproof

\usepackage{thmtools}
\usepackage{xcolor}
\AtBeginDocument{%
  \LetLtxMacro\proof\amsproof
  \LetLtxMacro\endproof\amsendproof
}

\newtheorem{theorem}{Theorem}
\newtheorem{proposition}[theorem]{Proposition}
\newtheorem{definition}[theorem]{Definition}
\newtheorem{lemma}[theorem]{Lemma}

\newtheorem{corollary}[theorem]{Corollary}

\newtheorem*{theorem*}{Theorem}
\newtheorem{remark}[theorem]{Remark}

\usepackage{xpatch}
\xpatchcmd{\proof}{\itshape}{\normalfont\proofnamefont}{}{}

\usepackage[most]{tcolorbox}

\usepackage[letterpaper,margin=1in]{geometry}

\usepackage[colorlinks=true, allcolors=blue]{hyperref}

\usepackage{bbm}

\usepackage{algorithm}
\usepackage{algpseudocode}

\let\qedsymbol\undefined
\newcommand{\qedsymbol}{$\blacksquare$}

\usepackage[
  backend=biber,
  style=alphabetic,
  maxnames=99,
  maxcitenames=2,
  maxalphanames=3,
  minalphanames=3
]{biblatex}

\newcommand{\abscite}[2][]{%
(#1\if\relax\detokenize{#1}\relax\else\space\fi\citeyear{#2})%
}

\title{A Sharp Dynamic Haussler Lemma}

\author{Shihao Shao and Dacheng Tao}
\affil{College of Computing and Data Science, Nanyang Technological University}
\affil{\texttt{\{shihao.shao, dacheng.tao\}@ntu.edu.sg}}

\date{}

\begin{document}

\maketitle
\pagestyle{empty}
\thispagestyle{empty}

\begin{abstract}
Haussler's covering lemma~\abscite[JCTA]{haussler1995sphere} established the sharp $(C/\delta)^d$ metric-entropy bound for binary classes of VC dimension $d$ under normalized $\ell_1$ distance. We establish a sharp global dynamic counterpart for bounded real-valued classes in every normalized $\ell_p$ metric, $1\le p<\infty$. For every such $p$, every integer $d\ge1$, and every $0<\varepsilon<1/4$, we prove the matching bounds
\begin{equation*}
  \left(\frac{c}{\varepsilon}\right)^{pd} \le \sup_{\substack{
  \mathcal X,\ T\ge1,\
  \varnothing\ne\mathcal F\subseteq[0,1]^{\mathcal X}:\\
  \operatorname{sfat}_{\varepsilon/32}(\mathcal F)\le d}}
N_p^{\mathrm{dyn}}(\mathcal F,\varepsilon,T) \le \left(\frac{C}{\varepsilon}\right)^{pd},
\end{equation*}
where $0<c\le1\le C$ are universal constants, independent also of $p$. The upper bound holds for each class and each horizon, and extends to all $0<\varepsilon<1/2$. The cover is a single finite family of deterministic causal experts, fixed independently of the instance sequence, that approximates every target on every sequence to normalized $\ell_p$ error at most $\varepsilon$. More strongly, for each $d\ge1$, a single countable binary class witnesses the lower bound for every $1\le p<\infty$ and every $0<\varepsilon<1/4$, at all sufficiently large horizons, even with real-valued covering experts.

For $p=2$, our theorem resolves the horizon-independent integral sequential $\ell_2$ covering problem of Rakhlin, Sridharan, and Tewari~\abscite[PTRF]{rakhlin2015sequential} in a stronger global form. For $p=1$, it yields the sharp $(C/\delta)^{\operatorname{Ldim}(\mathcal H)}$ global dynamic covering bound for binary classes. This resolves the global approximate covering problem posed by Alon, Ben-Eliezer, Dagan, Moran, Naor, and Yogev~\abscite[STOC]{alon2021adversarial}, as well as the dynamic Haussler problem posed by Malliaris and Moran~\abscite[Trans. AMS]{malliaris2021agnostic}. Our proof starts with SOA-style predictions and scheduled rank-decreasing version-space updates to control large pointwise errors. Rather than keeping every resulting prediction stream as a standalone expert, we use accurate prefixes that can be extended by exact predictions after identifying a target. We bound the number of activated targets compatible with each prefix, but assigning their continuations naively could still concentrate demand on a few prefixes. Each activated target has many eligible prefixes but needs only one continuation. Shared-parent routing distributes these requests across eligible prefixes, avoiding such concentration and eliminating the logarithmic loss.
\end{abstract}

\clearpage
\tableofcontents
\clearpage
\pagestyle{plain}
\pagenumbering{arabic}

\section{Introduction}
\label{sec:introduction}

How many predictors are needed to approximate an entire function class? In the classical setting, Haussler's covering lemma~\cite{haussler1995sphere} gives a sharp answer for binary classes. For every class $\mathcal H\subseteq\{0,1\}^{\mathcal X}$ of VC dimension $d$, every finite instance sequence $x_1,\ldots,x_T$, and every $0<\delta<1$, there is a set of at most $(C/\delta)^d$ representatives such that each target in $\mathcal H$ disagrees with some representative on at most $\delta T$ positions. The bound is independent of the sequence length, and the exponent is exactly the combinatorial dimension. VC dimension originated in uniform-convergence theory~\cite{vapnik1971uniform} and became a central parameter in distribution-free PAC learnability~\cite{blumer1989learnability}. The connection between combinatorial dimensions and metric entropy extends to real-valued function classes~\cite{mendelson2003entropy} and plays a central role in uniform convergence and generalization~\cite{van1996weak,bartlett2002rademacher}. More broadly, metric entropy underlies chaining bounds for stochastic processes~\cite{dudley1967sizes,talagrand2005generic}.

The classical cover may be chosen after the entire instance sequence is known. Such hindsight dependence is unsuitable for online prediction, where future instances are unavailable. To eliminate this dependence, one may instead maintain a bounded number of candidate prediction streams, updating their predictions as instances arrive. This leads to the notion of a \emph{causal expert}, whose prediction at time $t$ may depend on the observed prefix $x_1,\ldots,x_t$. We ask whether a single finite family of such experts, fixed in advance and consistent across shared input prefixes, can approximate every target on every possible input sequence. Moreover, finite VC dimension alone does not suffice for a horizon-independent global cover, even with causal experts. Indeed, threshold classes on the real line have VC dimension one but infinite Littlestone dimension~\cite{alon2021adversarial}, and a standard shattered-tree counting argument forces exponentially many experts for any fixed normalized $\ell_1$ error below $1/4$~\cite{ben2009agnostic}.

We study this question for bounded real-valued classes, measuring approximation in normalized $\ell_p$ distance for any $1\le p<\infty$. Let $\mathcal F\subseteq[0,1]^{\mathcal X}$ be nonempty, and fix a horizon $T\ge1$. A deterministic causal expert is a tuple $g=(g_t)_{t=1}^T$ with $g_t:\mathcal X^t\to[0,1]$. A finite family $\mathcal G$ is a global dynamic $\ell_p$ cover at accuracy $\varepsilon$ if
\begin{equation}
  \label{eq:intro-global-cover}
  \forall x_{1:T}\in\mathcal X^T\; \forall f\in\mathcal F\; \exists g\in\mathcal G: \quad \left(\frac1T\sum_{t=1}^T |g_t(x_{1:t})-f(x_t)|^p\right)^{1/p} \le\varepsilon.
\end{equation}
We write $N_p^{\mathrm{dyn}}(\mathcal F,\varepsilon,T)$ for the smallest cardinality of such a family. The family may depend on $\mathcal F,p,\varepsilon$, and $T$, but is fixed before the instances are revealed. Its experts receive no target values, and the expert witnessing the approximation may depend on both the target and the entire sequence. Thus \emph{global} refers to a single family that works on all instance sequences.

Sequential covering has a substantial history. The Littlestone dimension characterizes the optimal deterministic mistake bound in realizable online learning~\cite{littlestone1988learning}. Extending this characterization to agnostic online learning, Ben-David, P\'al, and Shalev-Shwartz~\cite{ben2009agnostic} constructed a global exact cover for every binary class of Littlestone dimension $d$, consisting of at most $\sum_{i=0}^{\min\{d,T\}}\binom{T}{i}$ deterministic causal experts. This exact cover has size growing with the horizon $T$. A natural question is whether allowing a positive approximation error can eliminate this dependence. In the real-valued setting, Rakhlin, Sridharan, and Tewari~\cite{rakhlin2010online,rakhlin2015online, rakhlin2015sequential} developed sequential covering numbers whose covering families may depend on the entire instance tree, and asked for an $\ell_2$ bound independent of the horizon. Related sequential complexity methods also apply to objectives beyond external regret, including calibration and approachability~\cite{rakhlin2011beyond}. Sequential covering numbers also have direct applications to online prediction. Rakhlin and Sridharan~\cite{rakhlin2014nonparametric} established optimal regret rates for online nonparametric regression in terms of sequential entropy. Gaillard and Gerchinovitz~\cite{gaillard2015chaining} subsequently developed chaining-based forecasters with entropy-dependent regret guarantees. The global approximate covering question for binary classes was posed in~\cite{alon2021adversarial}, and Malliaris and Moran~\cite{malliaris2021agnostic} explicitly asked for a dynamic analogue of Haussler's lemma. Horizon-independent fractional covering bounds were established for global binary covering by Alon, Ben-Eliezer, Dagan, Moran, Naor, and Yogev~\cite{alon2021adversarial} and for tree-dependent real-valued sequential covering by Block, Dagan, and Rakhlin~\cite{block2021majorizing}. The former provides one distribution over causal experts, whereas the latter allows the distribution to depend on the prescribed instance tree. Such \emph{fractional} covering guarantees do not directly yield a finite family of comparable size. In contrast, an \emph{integral} cover requires an actual finite family of experts. Haussler's lemma provides such a cover in the static setting, but a sharp, horizon-independent integral dynamic analogue, together with its extension to real-valued classes, has remained open.

\subsection{Our Contributions}
\label{subsec:our-contributions}

We establish a sharp global dynamic covering theorem for bounded real-valued classes, uniformly over all finite $p\ge1$. The controlling parameter is the \emph{sequential fat-shattering dimension} $\operatorname{sfat}_{\eta}(\mathcal F)$, the scale-sensitive sequential analogue of VC dimension. Informally, it is the maximum depth of a complete binary instance tree on which every path can be realized by a function in $\mathcal F$, with the two alternatives at each node separated by a gap $\eta$ around a node-dependent threshold. For binary classes, this dimension equals Littlestone dimension at every scale $0<\eta\le1$. The precise definition is given in the preliminaries.

\begin{theorem}[Sharp global dynamic $\ell_p$ covering]
\label{thm:intro-main}
There are universal constants $0<c\le1\le C$ such that, for every $1\le p<\infty$, every integer $d\ge0$, and every $0<\varepsilon<1/4$,
\begin{equation}
  \label{eq:intro-bounds}
  \left(\frac{c}{\varepsilon}\right)^{pd} \le \sup_{\substack{ \mathcal X,\ T\ge1,\ \varnothing\ne\mathcal F\subseteq[0,1]^{\mathcal X}:\\
  \operatorname{sfat}_{\varepsilon/32}(\mathcal F)\le d }} N_p^{\mathrm{dyn}}(\mathcal F,\varepsilon,T) \le \left(\frac{C}{\varepsilon}\right)^{pd}.
\end{equation}
More precisely, for every $1\le p<\infty$, every $T\ge1$, and every $0<\varepsilon<1/2$, each nonempty class $\mathcal F\subseteq[0,1]^{\mathcal X}$ with $\operatorname{sfat}_{\varepsilon/32}(\mathcal F)<\infty$ satisfies
\begin{equation}
  \label{eq:intro-lp}
  N_p^{\mathrm{dyn}}(\mathcal F,\varepsilon,T) \le \left(\frac{C}{\varepsilon^p}\right)^%
  {\operatorname{sfat}_{\varepsilon/32}(\mathcal F)}.
\end{equation}
Moreover, for each integer $d\ge1$, there is a countable binary class $\mathcal H_d$ with $\operatorname{Ldim}(\mathcal H_d) =\operatorname{sfat}_{\gamma}(\mathcal H_d)=d$ for every $0<\gamma\le1$, fixed independently of $p$ and $\varepsilon$, such that, for every $1\le p<\infty$ and $0<\varepsilon<1/4$, there is an integer $T_0=T_0(d,p,\varepsilon)\ge1$ for which
\begin{equation}
  \label{eq:intro-fixed-class}
  N_p^{\mathrm{dyn}}(\mathcal H_d,\varepsilon,T)>\left(\frac{1}{4\varepsilon}\right)^{pd} \qquad(T\ge T_0).
\end{equation}
For each $p,\varepsilon$, one fixed infinite periodic instance sequence witnesses this lower bound on every prefix of length $T\ge T_0$, even when the covering experts are allowed real-valued predictions. The constants are independent of $p$, $d$, $T$, and $\varepsilon$. If $\operatorname{sfat}_{\varepsilon/32}(\mathcal F)=0$, one expert suffices.
\end{theorem}

The theorem establishes the sharp worst-case dependence on the accuracy parameter $\varepsilon$, with exponent $pd$ and universal constants in the base. The more precise upper bound \eqref{eq:intro-lp} retains $d$ as the outer exponent and has a base constant independent also of $p$. Thus the bounds for $p=1$ and $p=2$ are $(C/\varepsilon)^d$ and $(C/\varepsilon^2)^d$, respectively. 

The class in \eqref{eq:intro-fixed-class} is the union of affine-hyperplane indicators on $\mathbb F_{2^m}^d$, $m\ge1$, each extended by zero outside its own block in a disjoint union of domains. Its Littlestone dimension stays $d$ across this union. For each $p$ and $\varepsilon$ in the stated range, choosing one block and listing its points once makes the subfamily supported on that block a normalized $\ell_p$ packing with separation greater than $2\varepsilon$. This strict separation survives periodic repetition on every sufficiently long prefix. Since the $p$th power of the normalized $\ell_p$ distance between binary vectors equals their normalized Hamming distance, the same class explains the exponent $pd$ for every finite $p$. The triangle inequality shows that even a real-valued expert can cover at most one vector in this packing.

\paragraph{Sequential covering and the $p=2$ problem.}
Restricting every global expert to a fixed instance tree gives a sequential cover on that tree of cardinality at most that of the global family. Theorem~\ref{thm:intro-main} therefore yields horizon-independent integral sequential $\ell_p$ covers for every $1\le p<\infty$. In particular, its $p=2$ case resolves the sequential $\ell_2$ covering question of Rakhlin, Sridharan, and Tewari~\cite{rakhlin2015sequential} in a stronger \emph{global} form: a single family of causal experts, fixed independently of the instance sequence, suffices to approximate every target on every possible sequence. It supplies an integral covering theorem alongside the fractional result of Block, Dagan, and Rakhlin~\cite{block2021majorizing}.

\paragraph{Binary classes and the $p=1$ case.}
For a binary class $\mathcal H$ of Littlestone dimension $d$, the $p=1$ case yields a family of at most $(C'/\delta)^d$ \emph{binary} causal experts such that
\begin{equation}
  \label{eq:intro-binary}
  \forall x_{1:T}\in\mathcal X^T\; \forall h\in\mathcal H\; \exists g\in\mathcal G: \quad \sum_{t=1}^T\mathbf 1\{g_t(x_{1:t})\ne h(x_t)\} \le\delta T,
\end{equation}
for every $0<\delta<1/4$. Indeed, apply \eqref{eq:intro-lp} with $p=1$ and $\varepsilon=\delta/2$, then round each prediction at $1/2$. Each binary mistake contributes at least $1/2$ to the original absolute error, so the number of mistakes is at most $\delta T$. The fixed class $\mathcal H_d$ gives a matching lower bound $(c'/\delta)^d$ at every sufficiently large horizon for each $0<\delta<1/4$. This settles the global dynamic covering question posed in~\cite{alon2021adversarial} and the dynamic Haussler problem of Malliaris and Moran~\cite{malliaris2021agnostic}. 

\paragraph{Proof ideas.}
The main technical step is a global outlier-cover proposition: for a uniformly separated class, a family of at most $(C_0R)^d$ causal experts guarantees that, for every target and instance sequence, some expert incurs pointwise error at least $8\eta$ on fewer than $T/(8R)$ rounds. The desired $\ell_p$ bound then follows by setting $\eta=\varepsilon/32$, passing to a maximal uniformly $\eta$-separated subclass of $\mathcal F$, and choosing $R=16/\varepsilon^p$. These choices yield mean $p$th-power error at most $(9\varepsilon/32)^p+\varepsilon^p/128<\varepsilon^p$ for the original class, giving the bound $(16C_0/\varepsilon^p)^d$ for every finite $p$.

Building on the SOA-based mistake-pattern constructions of Ben-David, Pál, and Shalev-Shwartz~\cite{ben2009agnostic} and their scale-sensitive real-valued analogues~\cite{rakhlin2010online}, we consider a target-dependent procedure that corrects its predictions through rank-decreasing version-space updates. Each correction decreases the sequential fat-shattering dimension, so at most $d$ corrections are needed. Although our experts cannot observe the target values, we can enumerate all possible correction schedules, specifying the ordinals and directions of hypothetical corrections. For every target and instance sequence, one such schedule reproduces the target-dependent procedure. Applying this construction on every round gives a global pointwise cover, but its size depends on the horizon $T$.

Following the sampling-based approach to dynamic covering in~\cite{alon2021adversarial}, we restrict corrections to Bernoulli-sampled rounds. By controlling the sampling rate, we obtain short correction schedules whose number is independent of $T$. However, a direct single-pool analysis incurs a factor of $(\log R)^d$ when ensuring simultaneous accuracy for all targets. To reduce this loss, we introduce $k=\Theta(\log R)$ independent Bernoulli sampling pools, each with expected size $O(dR)$ and a cap of $O(dR)$ selected rounds, together with a shared master sampling stream formed from their union (or all rounds when this is more economical). For each target, an analytical master runs on this stream and serves as a common reference for the individual pools. We first control the master's prediction errors against the target. Each pool then performs a tolerant replay using the master's predictions on its sampled rounds. A joint disagreement argument ensures that most of these replays remain close to the master on all but a small number of rounds. Combining the two error guarantees shows that every target has accurate prediction histories in a constant fraction of the pools. Retaining all pool schedules would require $(CR)^d\log R$ experts, reducing the logarithmic loss from $(\log R)^d$ to $\log R$. 

To eliminate the remaining logarithmic factor, we exploit the redundancy among candidate experts generated by different pool–schedule pairs. Rather than retaining all full correction schedules as experts, we retain only schedules of depth at most $d-1$ as \emph{parents}, each equipped with preallocated \emph{children} that can switch to exact predictions of an identified function. A function \emph{activates} once it has been robustly identified in at least half the pools. With high probability, the master-accuracy and joint replay-disagreement guarantees hold simultaneously, ensuring that every activated function has at least \(k/4\) eligible parents with accurate prefixes, while each parent can be eligible for only a bounded number of activations. Randomized shared-parent routing distributes these continuation requests among eligible parents, reducing the required number of children per parent by a factor proportional to $k$. This cancels the overhead from maintaining multiple pools. Functions that never activate are covered directly by their accurate parent schedules. The reduced $d-1$ depth of the parents and the balanced allocation of children together yield the sharp $(CR)^d$ expert count. Finally, we show that the accuracy and routing guarantees hold simultaneously with positive probability against every adaptive input strategy. A finite-horizon perfect-information game argument then converts this randomized guarantee into a deterministic global cover of the same size.

\subsection{Historical Context and Related Work}
\label{subsec:historical-context}

Three distinctions organize the comparison with earlier
covering results: whether the cover is chosen after seeing
an instance sequence or tree, whether the guarantee is
integral or fractional, and whether its size depends on
the horizon. Table~\ref{tab:covering-comparison}
summarizes the principal bounds.

\begin{table}[H]
\centering
\caption{Comparison of covering guarantees.
Here $d_{\mathrm{VC}}=\operatorname{VCdim}(\mathcal H)$,
$d=\operatorname{Ldim}(\mathcal H)$ for binary dynamic
classes, and $\mathrm{sfat}_{\eta}$ abbreviates
$\operatorname{sfat}_{\eta}(\mathcal F)$.
The parameter $\delta$ denotes normalized Hamming error,
while $\Gamma$ is the reciprocal guaranteed success mass
for fractional covers. Constants may vary across rows.}
\label{tab:covering-comparison}
\setlength{\tabcolsep}{3pt}
\renewcommand{\arraystretch}{1.13}
\begin{tabular}{@{}llll@{}}
\toprule
Work & Class / metric & Cover selection & Size upper bound \\
\midrule
\cite{haussler1995sphere}
  & Binary, Hamming
  & $\forall S\,\exists V_S$ (int.)
  & $(C/\delta)^{d_{\mathrm{VC}}}$ \\
\cite{ben2009agnostic}
  & Binary, exact
  & $\exists\mathcal G\,\forall S$ (int.)
  & $\sum_{i=0}^{\min(d,T)}\binom{T}{i}$ \\
\cite{rakhlin2015sequential}
  & Real, $\ell_\infty$
  & $\forall\mathbf x\,\exists V_{\mathbf x}$ (int.)
  & $(2eT/\varepsilon)^{\mathrm{sfat}_{\varepsilon}}$ \\
\cite{alon2021adversarial}
  & Binary, Hamming
  & $\exists\mu\,\forall S$ (frac.)
  & $\Gamma\le(C/\delta)^d$ \\
\cite{block2021majorizing}
  & Real, $\ell_2$
  & $\forall\mathbf x\,\exists\mu_{\mathbf x}$ (frac.)
  & $\Gamma\le(C/\varepsilon)^{3\mathrm{sfat}_{c\varepsilon}}$ \\
\textbf{This work}
  & \textbf{Real, $\ell_p$ ($1\le p<\infty$)}
  & $\exists\mathcal G\,\forall S$ (int.)
  & $N\le(C/\varepsilon^p)^{\mathrm{sfat}_{\varepsilon/32}}$ \\
\bottomrule
\end{tabular}
\end{table}

\paragraph{Cover-selection quantifiers.}
The quantifier column records when the covering family
or measure is selected, suppressing the subsequent
quantifiers over targets and approximation witnesses.
In Haussler's classical theorem~\cite{haussler1995sphere},
the representatives may be chosen after the sample $S$
is specified. In sequential covering, the family
$V_{\mathbf x}$ may depend on the entire prescribed
instance tree $\mathbf x$, including nodes outside the
realized path. A global cover instead fixes one family
of causal experts before any instance sequence is chosen.

For integral sequential and global covers, the distinction
can be written explicitly as
\begin{equation}
  \label{eq:related-quantifiers}
  \begin{aligned}
    \text{tree-dependent:}\quad
    &\forall\,\mathbf x\;\exists\,V_{\mathbf x}\;
      \forall f\in\mathcal F\;
      \forall\sigma\in\{-1,+1\}^T\;
      \exists\mathbf v\in V_{\mathbf x},\\
    \text{global:}\quad
    &\exists\,\mathcal G\;
      \forall x_{1:T}\in\mathcal X^T\;
      \forall f\in\mathcal F\;
      \exists g\in\mathcal G.
  \end{aligned}
\end{equation}
The approximation requirement in the first line is
evaluated along the path $\sigma$ of $\mathbf x$;
the second line uses~\eqref{eq:intro-global-cover}.

Restricting a global family to any fixed instance tree
gives a sequential cover of no greater size. Conversely,
taking a supremum over tree-dependent covering numbers
does not interchange the first two quantifiers.
Fractional covering replaces the choice of a finite
family by a probability measure with uniformly positive
mass on successful approximations, as formalized below.
Throughout, horizon-independent bounds concern the
size of the cover; the family itself may depend on $T$.

These questions also connect online learning to model-theoretic stability. Malliaris and Shelah~\cite{malliaris2014regularity} established regularity lemmas for stable graphs, while Malliaris and Moran~\cite{malliaris2021agnostic} subsequently studied excellent sets through agnostic online learning.

\paragraph{Real-valued metric entropy.}
Scale-sensitive dimensions have long been used to control static covering numbers for real-valued function classes, both on finite samples and under distribution-dependent metrics~\cite{alon1997scale,bartlett1997covering,mendelson2003entropy}. Empirical entropy also yields oracle inequalities and risk bounds for nonparametric regression~\cite{rakhlin2017empirical}. Recently, Aiyer, Mansour, Moran, Shao, and Waknine~\cite{aiyer2026scale} obtained sharp results for empirical $\ell_\infty$ covering numbers at optimal scales. In these settings, covering families may depend on the sample or probability measure defining the metric. Our theorem instead provides a single finite family of causal experts, with sharp global dynamic bounds in every finite $\ell_p$ metric.

\paragraph{Geometric refinements.}
Haussler's packing lemma has also been influential in discrete and computational geometry. Its shallow-packing refinements yield bounds for geometric set systems, with applications to geometric discrepancy, low-crossing spanning trees, and $\varepsilon$-nets~\cite{dutta2015shallow,dutta2019shallow}. These results exploit additional structure of static set systems, whereas our focus is sharp global causal covering.

\paragraph{Fractional and integral covering.}
For a target--sequence pair $(f,x_{1:T})$, let $B_{\varepsilon}(f,x_{1:T})$ be the set of causal experts satisfying the error bound in \eqref{eq:intro-global-cover} for that pair. A global fractional cover of size $\Gamma\ge1$ is a probability measure $\mu$, fixed independently of the input sequence, such that
\begin{equation}
  \label{eq:related-fractional}
  \forall x_{1:T}\in\mathcal X^T\;\forall f\in\mathcal F: \qquad \mu\bigl(B_{\varepsilon}(f,x_{1:T})\bigr)\ge\Gamma^{-1}.
\end{equation}
This is the probability-mass formulation used for binary dynamic covers in~\cite{alon2021adversarial}. A tree-dependent fractional cover instead chooses a measure $\mu_{\mathbf x}$ on prediction trees for each fixed $\mathbf x$, and imposes the same mass condition for every target and path on that tree, as in Block, Dagan, and Rakhlin~\cite{block2021majorizing}.

An integral cover requires an actual finite family $\mathcal G$ intersecting every set $B_{\varepsilon}(f,x_{1:T})$. The uniform distribution on a family of size $K$ satisfies \eqref{eq:related-fractional} with $\Gamma=K$. The converse is not immediate: $\Gamma$ controls the reciprocal of the mass of each successful set, rather than the cardinality of the support of $\mu$. In particular, even a finitely supported fractional cover need not have support of size at most $\Gamma$. Here \emph{integral} concerns the selection of representatives; the representatives themselves may make real-valued predictions. Such a finite family can also serve as a reference class in prediction-with-expert-advice methods~\cite{cesabianchi2006prediction}, although our results do not address efficient construction of the experts.

The horizon-independent fractional bounds summarized in
Table~\ref{tab:covering-comparison} control the reciprocal
success mass $\Gamma$, but do not directly yield finite
global integral covers of comparable size. The difficulty is simultaneous coverage. Sampling $m$ independent experts from $\mu$ misses a fixed target--sequence pair with probability at most $(1-1/\Gamma)^m\le e^{-m/\Gamma}$. For $M$ prescribed pairs, a union bound gives simultaneous coverage with probability at least $1-\zeta$ using $m=\lceil\Gamma\log(M/\zeta)\rceil$ samples, where $0<\zeta<1$. Thus, for any prescribed finite collection of $M$ target--sequence pairs, there exists a family of $m$ experts covering all of them. However, this argument introduces a logarithmic factor and does not directly handle the possibly infinite collection of all target--sequence pairs. On a fixed static sample, the classical packing argument converts fractional mass into a finite cover after doubling the radius. In the sequential setting, however, maximal packings on individual paths need not yield a common covering family, while packing uniformly across paths imposes a stronger requirement that is not controlled by the fractional covering bound. Block, Dagan, and Rakhlin~\cite{block2021majorizing} discuss this failure of the classical packing--covering argument.

Our proof instead uses shared-parent routing to share continuation capacity across compatible sampled histories. This yields a sharp bound on the global integral covering number in terms of sequential fat-shattering dimension, without relying on a general fractional-to-integral conversion principle.

\section{Technical Overview}
\label{sec:overview}
\label{subsec:overview-reduction}

The upper-bound proof reduces to establishing a global \emph{outlier-cover} proposition. Rather than requiring uniformly small prediction error on every round, this proposition allows a small number of rounds with larger errors. This relaxation is sufficient for the normalized $\ell_p$ bound and is what makes a horizon-independent expert count possible. We call a round with error at least $8\eta$ an \emph{outlier}, and write $\|h-h'\|_\infty=\sup_{x\in\mathcal X}|h(x)-h'(x)|$.

\begin{proposition}[Global outlier cover]
\label{prop:core}
There is a universal constant $C_0\ge1$ with the following property. Let $d\ge0$ be an integer, $0<\eta<1$, $R>32$, $T\ge1$, and let $\varnothing\ne\mathcal H\subseteq[0,1]^{\mathcal X}$ satisfy $\operatorname{sfat}_{\eta}(\mathcal H)\le d<\infty$ and $\|h-h'\|_\infty>\eta$ for distinct $h,h'\in\mathcal H$. There is a family $\mathcal G$ of at most $(C_0R)^d$ deterministic causal experts such that, for every instance sequence $x_{1:T}$ and every $h\in\mathcal H$, some $g\in\mathcal G$ satisfies
\begin{equation}
  \label{eq:overview-outliers}
  \#\{t\in[T]:|g_t(x_{1:t})-h(x_t)|\ge8\eta\} <\frac{T}{8R}.
\end{equation}
The family is fixed independently of the instance sequence.
\end{proposition}

This proposition directly implies the $\ell_p$ upper bound. Set $\eta=\varepsilon/32$, $d=\operatorname{sfat}_{\eta}(\mathcal F)$, and $R=16/\varepsilon^p$. Choose a maximal subclass $\mathcal H\subseteq\mathcal F$ whose distinct functions have uniform distance greater than $\eta$. Maximality ensures that every $f\in\mathcal F$ is within uniform distance $\eta$ of some $h\in\mathcal H$, and $\operatorname{sfat}_{\eta}(\mathcal H)\le d$. Apply Proposition~\ref{prop:core} to $h$. Outside its outliers, the expert's error against $f$ is below $9\eta$; on the remaining rounds it is at most one. Thus
\begin{equation}
  \label{eq:overview-lp-reduction}
  \frac1T\sum_{t=1}^T|g_t(x_{1:t})-f(x_t)|^p
  \le(9\eta)^p+\frac1{8R}
  =\varepsilon^p\left[\left(\frac9{32}\right)^p+\frac1{128}\right]<\varepsilon^p.
\end{equation}
The same family has at most $(16C_0/\varepsilon^p)^d$ experts. Hence the main challenge is to prove Proposition~\ref{prop:core}; the dependence on $p$ in the reduction enters only through the allowed number of outliers.

Our route to this proposition begins with a simple cover whose size depends on $T$. We then restrict corrections to sampled rounds, called a \emph{pool}, and progressively eliminate the logarithmic overhead introduced by the sampling-based construction. The main counting improvements are
\begin{equation*}
  \underbrace{(CR\log R)^d}_{\text{one long pool}}
  \quad\longrightarrow\quad
  \underbrace{(CR)^d\log R}_{\text{many short pools}}
  \quad\longrightarrow\quad
  \underbrace{(CR)^d}_{\text{shared-parent routing}}.
\end{equation*}
Here the first two expressions arise from retaining candidate prediction streams specified by a sampling pool and a correction schedule; the last exploits the fact that different such candidates can serve the same target. The analysis uses four favorable events, denoted $\mathsf{Cap}$, $\mathsf{Master}$, $\mathsf{Joint}$, and $\mathsf{NoOverload}$, which we introduce when their roles arise. We establish simultaneous success with positive probability even when future instances react to earlier random choices. A final finite-horizon game argument then yields a deterministic family that works on every instance sequence.

\subsection{Rank-Decreasing Correction Schedules}

\paragraph{Corrections that decrease dimension.}
A \emph{version space} $V\subseteq\mathcal H$ is the set of candidate functions remaining after previously imposed restrictions. Its \emph{rank} is $\operatorname{sfat}_{\eta}(V)$, with rank $-1$ for the empty set. For a nonempty $V$ of rank $r$ and current instance $x$, an SOA-style rule predicts a center $a(V,x)\in[0,1]$ such that both sign-update classes
\begin{equation}
  \label{eq:overview-rank-drop}
  V_x^-:=\{f\in V:f(x)\le a(V,x)-\eta\},\qquad
  V_x^+:=\{f\in V:f(x)\ge a(V,x)+\eta\}
\end{equation}
have rank at most $r-1$. If both had rank $r$, attaching the corresponding shattered trees below a root at $x$ would give a tree of depth $r+1$; Lemma~\ref{lem:center} specifies a valid center. Because $\mathcal H$ is uniformly greater-than-$\eta$ separated, every nonempty rank-zero version space is a singleton and can predict exactly. The case $d=0$ is therefore immediate; assume $d\ge1$ and $|\mathcal H|\ge2$ below.

Imagine temporarily that the predictor can observe the target function $h$. At a selected round, it compares $a(V,x_t)$ with $h(x_t)$ and, whenever their difference exceeds $\eta$, restricts $V$ to the sign class containing $h$, recomputes its center, and repeats. Each restriction retains $h$ and strictly decreases rank. Thus at most $d$ corrections occur over the whole sequence, and the final prediction on each selected round is within $\eta$ of $h(x_t)$. This \emph{target-dependent reference procedure} explains which updates would be useful; the experts themselves will not observe $h(x_t)$.

\paragraph{Encoding updates without target values.}
A selection list is an ordered list of rounds on which corrections are permitted; the expert still predicts on every round, but makes no correction on unselected rounds. If the list has at most $m$ positions, we pre-enumerate \emph{correction schedules}
\begin{equation*}
  J=((s_1,\sigma_1),\ldots,(s_i,\sigma_i)),\qquad
  1\le s_1\le\cdots\le s_i\le m,\quad
  \sigma_\ell\in\{-,+\},\quad i\le d.
\end{equation*}
Each $s_\ell$ is an \emph{ordinal} in the selected list, not an absolute round in $[T]$, and $\sigma_\ell$ specifies the direction of the hypothetical correction. Repeated ordinals allow several corrections at one instance. Each schedule prescribes its updates before emitting the prediction on a selected round; an invalid update enters an absorbing state with output zero. The empty schedule is included. Recording the ordinals and signs of the target-dependent reference procedure gives, for every target and realized input-and-selection history, one schedule reproducing that procedure. Thus the family of correction schedules is fixed independently of target labels, while for each target and realized input-and-sampling history, some schedule in the family reproduces the corrections of the target-dependent procedure.

There are $n_i(m)=2^i\binom{m+i-1}{i}$ schedules of depth $i\ge1$, and $n_0(m)=1$. Their total through depth $r$ satisfies
\begin{equation}
  \label{eq:overview-schedule-count}
  Q_r(m):=\sum_{i=0}^r n_i(m)
  \le\left[2e\left(1+\frac mr\right)\right]^r\quad(r\ge1).
\end{equation}
Selecting every round gives a deterministic pointwise cover of size $Q_d(T)$, but this grows with $T$. It already handles short horizons $T=O(dR)$ at cost $(CR)^d$. The remaining task is to obtain comparable accuracy for larger $T$ while encoding corrections on much shorter lists.

\subsection{From One Pool to Many Short Pools}
\label{subsec:overview-sampling}

A \emph{Bernoulli sampling pool} selects each round independently with some probability $\gamma$, after the current instance $x_t$ is chosen. It allows corrections only on selected rounds. The instance sequence may adapt to past sampling choices, but not to the current sampling bits.

\paragraph{Why one pool incurs a logarithmic loss.}
Put $b=T/R$; Proposition~\ref{prop:core} permits fewer than $b/8$ outliers. A single pool with $\gamma b=\Theta(d)$ and a sample-length cap $m=\lceil2\gamma T\rceil=\Theta(dR)$ admits only $Q_d(m)\le(CR)^d$ schedules. But accuracy on the sampled rounds does not by itself imply that a schedule makes few errors on all rounds. The adaptive low-hit sampling lemma, Lemma~\ref{lem:lowhit}, bounds the chance of many large errors with too few sampled hits. Applying it to each fixed schedule and taking a union bound produces a principal term of the form
\begin{equation}
  \label{eq:overview-low-hit}
  \exp\!\left\{O\!\left(d\log\left(1+\frac{\gamma T}{d}\right)\right)-c\gamma b\right\},
\end{equation}
plus smaller terms controlling excessive sample lengths. With $\gamma b=\Theta(d)$, this estimate becomes ineffective as $R$ grows. Taking $\gamma b=A d\log R$ for a sufficiently large constant $A$ makes the union bound work, but increases the sample length to $\Theta(dR\log R)$ and the schedule count to $(CR\log R)^d$. When the required sampling rate is too large, selecting all rounds handles that regime. Thus a direct single-pool construction removes dependence on $T$ at the cost of $(\log R)^d$.

\paragraph{Short pools and a shared master sampling stream.}
To reduce this loss, use $k=\Theta(\log R)$ independent pools, each with expected size $O(dR)$ and a cap of $O(dR)$ selected rounds. Fix a sufficiently large universal constant $\alpha$. The all-round cover above handles $b\le32\alpha d$; otherwise set
\begin{equation}
  \label{eq:overview-parameters}
  k=4\lceil\log_2R\rceil,\qquad
  \beta=\frac{\alpha d}{b}<\frac1{32},\qquad
  m=\lceil2\beta T\rceil=\Theta(dR).
\end{equation}
Let $\xi_{t,j}$ be independent Bernoulli-$\beta$ indicators for round $t$ and pool $j$. The \emph{raw} pool is the complete list of its selected rounds; its \emph{capped} version retains only the first $m$ selections, ensuring that every formal schedule has at most $m$ selected positions. The event $\mathsf{Cap}$ means that none of the raw pools exceeds this cap. A Chernoff bound makes $\mathsf{Cap}$ highly probable; on this event, raw and capped simulations coincide.

The individual pools need a common reference against which to compare their predictions. We therefore form a \emph{master sampling stream} that contains all their selected rounds: it is the union of the raw pools when $k\beta\le1/8$, and all rounds when $k\beta>1/8$. In the latter case, $T<8\alpha kdR\le4km$, so this all-round master still has a controlled schedule length. For each target $h$, an \emph{analytical master} runs the target-dependent correction procedure on this shared stream. Thus the sampling stream is common, but the master's predictions depend on $h$; the master is a proof device, not an additional covering expert. Write $a_t^h$ for its prediction before the current-round corrections and $y_t^h$ for the prediction after them. Since each pool-selected round is master-selected,
\begin{equation*}
  \xi_{t,j}=1\quad\Longrightarrow\quad |y_t^h-h(x_t)|\le\eta.
\end{equation*}
On $\mathsf{Cap}$, every target's master corrections can be encoded by a schedule $J^h$ drawn from a fixed family $\mathcal J_*$ of $Q_d(4km)$ possible codes. This fixed-code representation is what makes uniform probability bounds possible despite the master's dependence on $h$.

Our second favorable event, $\mathsf{Master}$, asserts that \emph{every} target has fewer than $b/32$ rounds on which its master's provisional prediction differs from the target by at least $2\eta$:
\begin{equation*}
  \mathsf{Master}:=\left\{\forall h\in\mathcal H:\
  \#\{t:|a_t^h-h(x_t)|\ge2\eta\}<b/32\right\}.
\end{equation*}
In the all-round master regime, each such error forces a correction. Since there are at most $d<b/32$ corrections in the regime under consideration, $\mathsf{Master}$ holds deterministically. In the union regime, the master selection rate is $u=1-(1-\beta)^k=\Theta(k\beta)$; the adaptive low-hit lemma and a union bound over fixed master codes control these errors uniformly over targets. 

\paragraph{Tolerant replays and a joint disagreement bound.}
For each target $h$ and each pool $j$, a \emph{tolerant replay} maintains its own version space and uses the master's final predictions $y_t^h$ as reference values. On a round selected by pool $j$, it repeatedly makes a sign correction whenever its current center differs from $y_t^h$ by more than $4\eta$, recomputing the center after each update; on an unselected round it does not update. The tolerance leaves room to retain the true target. For instance, $y_t^h>a+4\eta$ implies $h(x_t)>a+3\eta$, whereas the corresponding update requires only $f(x_t)\ge a+\eta$. Thus each retained constraint has more than $2\eta$ of slack for $h$. Every replay makes at most $d$ corrections and has selected-round error at most $5\eta$. Although these replays depend on the target, each can be reproduced by a fixed correction schedule on the corresponding pool when $\mathsf{Cap}$ holds.

We must also control what happens on unselected rounds. For each \emph{fixed} master code $J\in\mathcal J_*$, let $a_t^J$ and $a_{t,j}^J$ denote the provisional predictions of its master simulation and its replay in pool $j$, respectively, both computed before the current sampling bits are revealed. Define
\begin{equation*}
  D_j^J:=\sum_{t=1}^T\mathbf1\{|a_t^J-a_{t,j}^J|>6\eta\},\qquad
  H_j^J:=\sum_{t=1}^T\xi_{t,j}\mathbf1\{|a_t^J-a_{t,j}^J|>6\eta\}.
\end{equation*}
Thus $D_j^J$ counts all significant master--replay disagreements, whereas $H_j^J$ counts the ones sampled by pool $j$. The event $\mathsf{Joint}$ rules out the possibility that, for some fixed $J$, at least $k/4$ pools \emph{all} have $D_j^J\ge b/16$ but $H_j^J\le2d$. Because the provisional disagreements are predictable before the fresh independent pool bits, a joint exponential supermartingale bounds this failure for any fixed $J$ and any prescribed $k/4$ pools by $e^{-\alpha kd/256}$. A union bound over $Q_d(4km)\binom{k}{k/4}$ code--subset pairs makes $\mathsf{Joint}$ highly probable.

Why does this event help for an actual target $h$? Each sampled disagreement $|a_t^h-a_{t,j}^h|>6\eta$ forces a correction in either the master or the replay, because otherwise the replay would have to correct toward the master's unchanged prediction. Each performs at most $d$ corrections, so $H_j^{J^h}\le2d$ on $\mathsf{Cap}$. Consequently, on $\mathsf{Joint}$, at least $3k/4$ pools satisfy $D_j^{J^h}<b/16$. Combining this with $\mathsf{Master}$ shows that, for each such pool, its replay has fewer than $b/32+b/16<b/8$ provisional errors of size at least $8\eta$. The same holds for its final predictions: on selected rounds their error is at most $5\eta$, and on unselected rounds the provisional and final predictions agree.

With $\alpha$ chosen sufficiently large, the three events just introduced hold simultaneously with probability greater than $15/16$ against every adaptive input rule. Denote their intersection by $\mathsf{Good}:=\mathsf{Cap}\cap\mathsf{Master}\cap\mathsf{Joint}$. We have therefore proved the simultaneous-accuracy guarantee of Proposition~\ref{prop:multi-pool-accuracy}: on $\mathsf{Good}$, every target has accurate replays in at least $3k/4$ pools. Keeping every full pool schedule as an expert would require only $kQ_d(m)\le(CR)^d\log R$ streams, but we must still remove the outer $\log R$.

\subsection{Shared-Parent Routing and Sharp Counting}
\label{subsec:overview-routing}

\paragraph{Robust identification enables exact continuation.}
\label{subsec:overview-identification}
The final saving comes from not retaining every complete pool--schedule prediction stream. A function $h$ can be identified when a schedule's surviving version space becomes the singleton $\{h\}$, subject to a robustness condition. Once sufficiently many pools identify $h$, we can assign its exact continuation to a child of an accurate parent schedule, which need not itself identify $h$. The parent's role is to provide an accurate prediction prefix, since correct future predictions cannot repair earlier errors.

A schedule is \emph{born} once all its prescribed corrections have been executed; from then on its version space no longer changes. A \emph{robust certificate} for $h$ is a valid born schedule whose version space is $\{h\}$ and whose correction constraints hold for $h$ with slack at least $2\eta$. The replay corrections above have precisely this margin. The slack ensures that two distinct functions robustly certified in the same pool differ by more than $2\eta$ on some selected instance. We say that a pool \emph{identifies} $h$ when it has such a certificate, and that $h$ \emph{activates} on the first round when at least $k/2$ pools identify it. The identification test concerns a candidate function known from a singleton version space; it does not require access to hidden target labels.

\paragraph{Parents preserve prefixes; children provide continuations.}
For each pool, retain as a \emph{parent} every correction schedule of depth at most $d-1$, rather than keeping all schedules through depth $d$ as covering experts. There are $P=kQ_{d-1}(m)$ parent indices. A valid born parent $v$ of depth $i$ has a fixed surviving version space $V_v$ of rank at most $d-i$. We preallocate $c_i$ \emph{children} to each such parent. A child initially copies the parent's predictions; if assigned to an activated function $h$, it switches to predicting $h$ exactly from the assignment round onward. All children are allocated in advance, so an assignment never changes their earlier outputs. Depth-$d$ schedules remain available internally for identification, but are not retained as parents.

Suppose $h$ activates on round $t$. A parent $v$ is \emph{eligible} if its schedule was valid and born by round $t-1$, its version space contains $h$ and is nonsingleton, and its predictions before $t$ have fewer than $b/8$ outliers against $h$:
\begin{equation*}
  h\in V_v,\qquad |V_v|\ge2,\qquad
  \sum_{s<t}\mathbf1\{|p_s^v-h(x_s)|\ge8\eta\}<b/8,
\end{equation*}
where $p_s^v$ is the parent's prediction. Eligibility can be checked using the identified function and the observed instance history, without access to target labels. Retain one eligible parent per pool in a fixed order, and denote the resulting set by $\mathcal E(h,t)$.

On $\mathsf{Good}$, at least $3k/4$ pools have accurate replays for $h$. Since $h$ activates for the first time at $t$, fewer than $k/2$ pools identified it by round $t-1$. Thus more than $k/4$ accurate replay prefixes still have nonsingleton version spaces; a singleton replay prefix would already supply a robust certificate. A nonsingleton space has positive rank, so each such prefix uses at most $d-1$ corrections and coincides with one of the preallocated parents. Hence $|\mathcal E(h,t)|\ge k/4$, and an unused child of any of these parents can cover $h$ by combining its accurate past with an exact continuation. A function that never activates is covered directly: one of its accurate, unidentified full replays has depth at most $d-1$ and is already a parent.

\paragraph{Robust certificates bound the requests per parent.}
The same certificate separation limits how many functions a pool can identify inside any rank-$r$ subclass $V$. On the pool's selected list, the pointwise radius-$\eta$ schedule cover of $V$ uses at most $Q_r(m)$ prediction streams. Since two distinct identified functions differ by more than $2\eta$ somewhere on that list, they cannot be approximated by the same stream. Thus each pool identifies at most $Q_r(m)$ members of $V$. If $\mathcal A$ is the set of activated functions, every member of $\mathcal A$ has been identified by at least $k/2$ pools. Double counting gives
\begin{equation}
  \label{eq:overview-local-count}
  \frac{k}{2}|\mathcal A\cap V|\le kQ_r(m),\qquad
  |\mathcal A\cap V|\le2Q_r(m).
\end{equation}
For a depth-$i$ parent $v$, this shows that at most $2Q_{d-i}(m)$ activated functions can regard $v$ as eligible. The bound holds for every realized history, including adaptively generated version spaces.

\paragraph{Routing spreads the continuation requests.}
If every activated function chose an arbitrary eligible parent, a depth-$i$ parent might need $2Q_{d-i}(m)$ children. Instead, whenever an activation has at least $k/4$ eligible parents, choose one uniformly using fresh routing randomness; an activation with fewer eligible parents is left unassigned. On $\mathsf{Good}$, the latter case never occurs. Each such request selects any specified eligible parent with conditional probability at most $4/k$. Combined with the deterministic request bound, this makes its expected load at most $8Q_{d-i}(m)/k$. Importantly, parents remain in the eligible set even if all their children have been used; we count any further selection as an overload rather than silently changing the routing probabilities.

An exponential load bound and a union bound over all $P$ parents show that it suffices to allocate
\begin{equation}
  \label{eq:overview-capacity}
  c_i=\left\lceil\frac{16Q_{d-i}(m)}{k}+\log(4P)\right\rceil,
  \qquad 0\le i\le d-1.
\end{equation}
Our fourth event, $\mathsf{NoOverload}$, is the event that every parent is selected at most as many times as its preallocated number of children. Its failure probability is at most $1/4$, uniformly over adaptive input rules, without conditioning on $\mathsf{Good}$. On $\mathsf{Good}\cap\mathsf{NoOverload}$, every activated function receives an unused child, while every unactivated function is already covered by a parent.

The resulting number of parent and child indices is
\begin{align}
  \label{eq:overview-expert-count}
  K&=P+\sum_{i=0}^{d-1}kn_i(m)c_i\notag\\
  &\le16\sum_{i=0}^{d-1}n_i(m)Q_{d-i}(m)
       +P\bigl(\log(4P)+2\bigr)\notag\\
  &\le16Q_d(2m)+P\bigl(\log(4P)+2\bigr).
\end{align}
The factor $k$ counting pools cancels against the $1/k$ in the leading child capacity. The last inequality is the schedule-count convolution bound in Lemma~\ref{lem:count}: a combined schedule of at most $d$ signed corrections can be distributed between two selected lists of length $m$. Since $m=O(dR)$, the first term is at most $(CR)^d$. Meanwhile, parents use only depths up to $d-1$, so their overhead is at most $Cd(\log R)^2(CR)^{d-1}$. The missing power of $R$ absorbs these logarithms (using $(\log R)^2\le R$ for $R>32$), and $d\le2^d$ is absorbed into the universal base constant. Thus $K\le(C'R)^d$, with no remaining logarithmic factor.

\subsection{Positive-Probability Success and Global Determinization}
\label{subsec:overview-verification}

The probabilistic construction has now supplied both parts of the outlier-cover guarantee. The event $\mathsf{Good}=\mathsf{Cap}\cap\mathsf{Master}\cap\mathsf{Joint}$ ensures accurate parents or sufficiently many eligible parents for every target, and holds with probability greater than $15/16$. The routing event $\mathsf{NoOverload}$ fails with probability at most $1/4$. A union bound, requiring no independence between these events, gives
\begin{equation*}
  \Pr(\mathsf{Good}\cap\mathsf{NoOverload})
  >1-\frac1{16}-\frac14=\frac{11}{16}>0.
\end{equation*}
Fix any deterministic adaptive input strategy $x_t=A_t(\boldsymbol\xi_{<t},\boldsymbol\pi_{<t})$, where $\boldsymbol\xi$ and $\boldsymbol\pi$ denote the sampling and routing choices. With probability greater than $11/16$ over these random choices, the resulting $K$ prediction streams simultaneously cover every target on the realized instance sequence.

The key remaining step is to turn this probabilistic guarantee against each adaptive adversary into a single deterministic strategy that succeeds against all adversaries. We use a finite-horizon perfect-information game in which the adversary chooses the current instance and the constructor then chooses that round's sampling and routing actions, after observing the instance and earlier moves. The constructor wins if its $K$ prediction streams cover every target. If no deterministic constructor strategy could win against every adversary, backward induction would yield a deterministic adaptive adversary that forces failure against every possible sequence of constructor actions, contradicting the positive-probability guarantee. Thus a deterministic winning constructor strategy exists. Fixing this strategy makes each prediction stream depend only on the observed instance prefix $x_{1:t}$, yielding a single global family of $K$ deterministic causal experts. Together with the expert count above, this proves Proposition~\ref{prop:core}, and the reduction in \eqref{eq:overview-lp-reduction} gives the general-$p$ upper bound.

\subsection{The Matching Lower Bound}
\label{subsec:overview-lower}

For each power of two $q$, take the hyperplane indicators $h_{a,c}(x)=\mathbf 1\{a\cdot x=c\}$ on $\mathbb F_q^d$, where $a\in\mathbb F_q^d\setminus\{0\}$ and $c\in\mathbb F_q$. This finite class contains $q(q^d-1)/(q-1)\ge q^d$ distinct functions. Extend each function by zero to the disjoint union of these domains over all such $q$, and let $\mathcal H_d$ be the union of the extensions. To bound its Littlestone dimension, predict one only on the affine span of previously observed positive instances in their common block. For any target in $\mathcal H_d$, the span stays inside its hyperplane, and each mistake increases its affine dimension from $-1$ toward at most $d-1$. Thus there are at most $d$ mistakes, while the $d$ standard basis vectors in a single block are shattered.

For $1\le p<\infty$ and $0<\varepsilon<1/4$, choose $q$ as the largest power of two not exceeding $(2\varepsilon)^{-p}$, and list all $L=q^d$ points in that block once. The normalized Hamming distance between distinct indicators supported on the block is at least $\delta_q:=2(q-1)/q^2>(2\varepsilon)^p$. Repeat this list periodically. In a prefix of length $T$, the complete periods give normalized Hamming distance at least $\lfloor T/L\rfloor L\delta_q/T$, which exceeds $(2\varepsilon)^p$ for every sufficiently large $T$. For binary vectors, this Hamming distance is the $p$th power of the normalized $\ell_p$ distance. Thus the traces are pairwise more than $2\varepsilon$ apart at all such horizons, and
\begin{equation*}
  q^d>\left(\frac1{2^{p+1}\varepsilon^p}\right)^d
  \ge\left(\frac1{4\varepsilon}\right)^{pd}.
\end{equation*}
The triangle inequality lets even an arbitrary real-valued covering vector approximate at most one trace. Since binary sequential fat-shattering dimension equals Littlestone dimension at every scale in $(0,1]$, this proves the matching exponent $pd$ for the same class $\mathcal H_d$ throughout the stated accuracy range and at every sufficiently large horizon, without using causality or adaptive inputs.

\section{Definitions and Main Results}
\label{sec:results}

We now formalize the notions used in the preceding overview and state the precise covering theorems and their consequences.

Throughout, $\mathcal X$ is a nonempty instance domain, and function classes are nonempty unless stated otherwise. For an integer $T\ge1$, write $[T]=\{1,\ldots,T\}$ and $x_{1:t}=(x_1,\ldots,x_t)$. For a sign sequence $\sigma$, write $\sigma_{<t}=(\sigma_1,\ldots,\sigma_{t-1})$; the prefix at $t=1$ is empty. The exponent $p$ always lies in $[1,\infty)$. All prediction rules are arbitrary set-theoretic maps; no computability or measurability assumption is imposed.

\subsection{Sequential Fat-Shattering Dimension}
\label{subsec:sfat}

Fat-shattering extends VC dimension to real-valued function classes~\cite{kearns1994efficient,bartlett1996fat}. Its sequential analogue is a central complexity measure in online learning~\cite{rakhlin2010online,rakhlin2015online}. An $\mathcal X$-valued binary tree of depth $r$ is a collection $\mathbf x=(\mathbf x_t)_{t=1}^r$ of maps $\mathbf x_t:\{-1,+1\}^{t-1}\to\mathcal X$. A real-valued tree $\mathbf s=(\mathbf s_t)_{t=1}^r$ is defined in the same way, with range $\mathbb R$. Thus the instance and threshold at level $t$ may depend on the preceding signs, but not on the sign selected at that level.

\begin{definition}[Sequential fat-shattering dimension]
\label{def:sfat}
Let $\eta>0$ and $\mathcal F\subseteq[0,1]^{\mathcal X}$. The class $\mathcal F$ \emph{sequentially fat-shatters} an instance tree $\mathbf x$ of depth $r$ at scale $\eta$ if there is a real-valued threshold tree $\mathbf s$ such that
\begin{equation}
  \label{eq:sfat-definition}
  \forall\sigma\in\{-1,+1\}^r\;\exists f\in\mathcal F\; \forall t\in[r]:\qquad \sigma_t\bigl(f(\mathbf x_t(\sigma_{<t})) -\mathbf s_t(\sigma_{<t})\bigr) \ge\frac{\eta}{2}.
\end{equation}
The \emph{sequential fat-shattering dimension} $\operatorname{sfat}_{\eta}(\mathcal F)$ is the largest such depth, and is $\infty$ if these depths are unbounded. A nonempty class has dimension at least zero; we set $\operatorname{sfat}_{\eta}(\varnothing)=-1$.
\end{definition}

Our convention uses a margin $\eta/2$ on each side of a threshold, so the two branches are separated by a gap $\eta$. Instances may repeat, including along a single path. The dimension is monotone under inclusion of classes and nonincreasing in the scale: if $\mathcal F_0\subseteq\mathcal F$ and $0<\eta\le\eta'$, then
\begin{equation*}
  \operatorname{sfat}_{\eta'}(\mathcal F_0) \le\operatorname{sfat}_{\eta}(\mathcal F).
\end{equation*}

For a binary class $\mathcal H\subseteq\{0,1\}^{\mathcal X}$, the \emph{Littlestone dimension} $\operatorname{Ldim}(\mathcal H)$ is the largest depth of an instance tree on which every sign path $\sigma$ has a witness $h\in\mathcal H$ satisfying $h(\mathbf x_t(\sigma_{<t}))=(1+\sigma_t)/2$ at every level. We use the same conventions for infinite dimension and the empty class. For binary classes, sequential fat-shattering dimension coincides with Littlestone dimension:
\begin{equation*}
  \operatorname{sfat}_{\eta}(\mathcal H) =\operatorname{Ldim}(\mathcal H), \qquad 0<\eta\le1.
\end{equation*}
Indeed, a Littlestone-shattered tree is $\eta$-fat-shattered with thresholds $1/2$, while fat-shattering of a binary class at any positive scale forces the two branches at each node to have labels $0$ and $1$, respectively.

\subsection{Global Dynamic Covering Numbers}
\label{subsec:dynamic-covers}

For $u,v\in\mathbb R^T$, define the normalized $\ell_p$ distance by
\begin{equation}
  \label{eq:normalized-lp}
  \rho_{p,T}(u,v) :=\left(\frac1T\sum_{t=1}^T|u_t-v_t|^p\right)^{1/p}.
\end{equation}
A \emph{deterministic causal expert} at horizon $T$ is a tuple $g=(g_t)_{t=1}^T$ with $g_t:\mathcal X^t\to[0,1]$. The prediction $g_t(x_{1:t})$ may use the current instance $x_t$ and all preceding instances. The expert receives no target values. For a fixed sequence, write
\begin{equation*}
  g[x_{1:T}]:=(g_t(x_{1:t}))_{t=1}^T, \qquad f[x_{1:T}]:=(f(x_t))_{t=1}^T.
\end{equation*}

\begin{definition}[Global dynamic $\ell_p$ cover]
\label{def:dynamic-cover}
Let $\mathcal F\subseteq[0,1]^{\mathcal X}$, $T\ge1$, and $\varepsilon>0$. A finite family $\mathcal G$ of deterministic causal experts is a \emph{global dynamic $\ell_p$ cover} of $\mathcal F$ at accuracy $\varepsilon$ and horizon $T$ if
\begin{equation}
  \label{eq:dynamic-cover}
  \forall x_{1:T}\in\mathcal X^T\; \forall f\in\mathcal F\; \exists g\in\mathcal G:\qquad \rho_{p,T}\bigl(g[x_{1:T}],f[x_{1:T}]\bigr) \le\varepsilon.
\end{equation}
The \emph{global dynamic covering number} $N_p^{\mathrm{dyn}}(\mathcal F,\varepsilon,T)$ is the least cardinality of such a finite family, or $\infty$ if none exists.
\end{definition}

The entire family, including its predictions on every possible prefix, is fixed before the sequence is chosen. The witnessing expert may depend on both the target and the complete sequence, but it is a single member of the family used throughout that sequence. There is no requirement that this member be identifiable online. Experts need not arise from functions in $\mathcal F$. In particular, \emph{integral} covering means selecting finitely many experts; it imposes no integrality constraint on their predictions.

The family may depend on $\mathcal F,p,\varepsilon$, and $T$. A bound independent of the horizon therefore asserts a uniform bound on its size for all $T$, while allowing a different family at each horizon.

\subsection{Sharp Upper and Lower Bounds}
\label{subsec:sharp-bounds}

\begin{theorem}[Global dynamic $\ell_p$ upper bound]
\label{thm:main}
There is a universal constant $C\ge1$ such that, for every $1\le p<\infty$, every integer $T\ge1$, every $0<\varepsilon<1/2$, and every nonempty $\mathcal F\subseteq[0,1]^{\mathcal X}$ with $d=\operatorname{sfat}_{\varepsilon/32}(\mathcal F)<\infty$,
\begin{equation}
  \label{eq:main-upper}
  N_p^{\mathrm{dyn}}(\mathcal F,\varepsilon,T) \le\left(\frac{C}{\varepsilon^p}\right)^d.
\end{equation}
The constant $C$ is independent of $p$, $d$, $T$, $\varepsilon$, and the class. If $d=0$, the covering number is exactly one.
\end{theorem}

The zero-dimensional case follows directly from the definition. If $\operatorname{sfat}_{\eta}(\mathcal F)=0$, no two values at any one instance can differ by at least $\eta$, since they would shatter a depth-one tree. Hence
\begin{equation*}
  \sup_{f\in\mathcal F}f(x)-\inf_{f\in\mathcal F}f(x)\le\eta \qquad(x\in\mathcal X).
\end{equation*}
The expert that predicts the midpoint of this interval at the current instance has pointwise error at most $\eta/2$ against every target. Taking $\eta=\varepsilon/32$ proves the assertion.

\paragraph{Necessity of finite $p$.}
The restriction $p<\infty$ is essential for a horizon-independent bound in terms of accuracy and sequential fat-shattering dimension. On $\mathcal X=\mathbb N$, the class $\mathcal H=\{0\}\cup\{\mathbf 1_{\{j\}}:j\in\mathbb N\}$ has Littlestone dimension one. On the sequence $x_t=t$, $t\in[T]$, its traces are the zero vector and the $T$ standard basis vectors, pairwise at $\ell_\infty$ distance one. For $0<\varepsilon<1/2$, the triangle inequality implies that a single real-valued prediction vector can be within $\ell_\infty$ distance $\varepsilon$ of at most one trace. Hence at least $T+1$ experts are required, even on this fixed sequence.

\begin{theorem}[Matching $\ell_p$ lower bound]
\label{thm:lower}
For every integer $d\ge1$, there are a countable domain $\mathcal X_d$ and a countable binary class $\mathcal H_d\subseteq\{0,1\}^{\mathcal X_d}$ with $\operatorname{Ldim}(\mathcal H_d)=\operatorname{sfat}_{\gamma}(\mathcal H_d)=d$ for all $0<\gamma\le1$, such that, for every $1\le p<\infty$ and every $0<\varepsilon<2^{-1-1/p}$, there are an integer $T_0=T_0(d,p,\varepsilon)\ge1$ and an infinite periodic instance sequence $(\bar x_t)_{t\ge1}$ in $\mathcal X_d$ for which every integer $T\ge T_0$ satisfies
\begin{equation}
  \label{eq:main-lower}
  N_p^{\mathrm{dyn}}(\mathcal H_d,\varepsilon,T) >\left(\frac{1}{2^{p+1}\varepsilon^p}\right)^d.
\end{equation}
For every such $T$, the lower bound holds even when covering is required only on the prefix $\bar x_{1:T}$ and the covering predictions are arbitrary vectors in $[0,1]^T$. The class and its domain depend only on $d$; $T_0$ and the periodic sequence may depend on $d$, $p$, and $\varepsilon$, but not on $T$.
\end{theorem}

The construction puts hyperplane indicators over every field $\mathbb F_{2^m}$ on disjoint blocks of one domain. For a suitable $m$, the $2^m$-block contains at least $(2^m)^d$ functions whose traces are pairwise more than $2\varepsilon$ apart in normalized $\ell_p$ distance when its points are listed once. Periodically repeating this list preserves the strict separation on every sufficiently long prefix. Each covering vector can therefore approximate at most one of these traces. The assertion concerns one fixed class of dimension $d$, at every accuracy in the theorem's range and every sufficiently large horizon, rather than every class of dimension $d$.

A fixed infinite domain is necessary for this asymptotic statement: on a finite domain every binary class is finite, and one causal expert per member gives $N_p^{\mathrm{dyn}}(\mathcal H,\varepsilon,T)\le|\mathcal H|$ for every accuracy and horizon.

For a nonnegative integer $d$, define the worst-case quantity
\begin{equation}
  \label{eq:worstcase}
  \mathfrak M_{d,p}(\varepsilon) :=\sup_{T\ge1}\; \sup_{\substack{ \mathcal X,\ \varnothing\ne\mathcal F\subseteq[0,1]^{\mathcal X}:\\
  \operatorname{sfat}_{\varepsilon/32}(\mathcal F)\le d }} N_p^{\mathrm{dyn}}(\mathcal F,\varepsilon,T).
\end{equation}
Theorems~\ref{thm:main} and~\ref{thm:lower} give, for $d\ge1$, $0<\varepsilon<2^{-1-1/p}$, and every $T\ge T_0(d,p,\varepsilon)$ from Theorem~\ref{thm:lower},
\begin{equation}
  \label{eq:sharp-bounds}
  \begin{aligned}
  \left(\frac{1}{2^{p+1}\varepsilon^p}\right)^d
  &<N_p^{\mathrm{dyn}}(\mathcal H_d,\varepsilon,T)
  \le\mathfrak M_{d,p}(\varepsilon)\le\left(\frac{C}{\varepsilon^p}\right)^d.
  \end{aligned}
\end{equation}
In particular, since $2^{1+1/p}\le4$ and $C^{1/p}\le C$, for every $1\le p<\infty$, $d\ge1$, $0<\varepsilon<1/4$, and $T\ge T_0(d,p,\varepsilon)$,
\begin{equation}
  \label{eq:sharp-universal-bases}
  \begin{aligned}
  \left(\frac{1}{4\varepsilon}\right)^{pd}
  &<N_p^{\mathrm{dyn}}(\mathcal H_d,\varepsilon,T)
  \le\mathfrak M_{d,p}(\varepsilon)\le\left(\frac{C}{\varepsilon}\right)^{pd}.
  \end{aligned}
\end{equation}
For $d=0$, $\mathfrak M_{0,p}(\varepsilon)=1$ on the same range. Together with Theorem~\ref{thm:main}, the preceding bounds prove Theorem~\ref{thm:intro-main} with $c=1/4$.

\subsection{Binary Classes and Sequential Covering}
\label{subsec:consequences}

We record two consequences with their proofs. The first gives binary-valued experts for binary classes; the second restricts the global family to a prescribed instance tree.

\paragraph{Binary Hamming covers.}
For $u,v\in\{0,1\}^T$, let
\begin{equation*}
  d_{\mathrm H,T}(u,v) :=\frac1T\sum_{t=1}^T\mathbf 1\{u_t\ne v_t\}
\end{equation*}
be normalized Hamming distance. For $\mathcal H\subseteq\{0,1\}^{\mathcal X}$, define $N_{\mathrm{bin}}^{\mathrm{dyn}}(\mathcal H,\delta,T)$ as in Definition~\ref{def:dynamic-cover}, with binary-valued experts $g_t:\mathcal X^t\to\{0,1\}$ and with $d_{\mathrm H,T}(g[x_{1:T}],h[x_{1:T}])\le\delta$ in place of the $\ell_p$ requirement. Requiring binary outputs distinguishes this number from $N_1^{\mathrm{dyn}}$, which permits real outputs.

\begin{corollary}[Sharp binary dynamic covering]
\label{cor:binary}
There is a universal constant $C_{\mathrm{bin}}\ge1$ such that, for every binary class $\mathcal H$ with $d=\operatorname{Ldim}(\mathcal H)<\infty$, every $T\ge1$, and every $0<\delta<1$,
\begin{equation}
  \label{eq:binary-upper}
  N_{\mathrm{bin}}^{\mathrm{dyn}}(\mathcal H,\delta,T) \le\left(\frac{C_{\mathrm{bin}}}{\delta}\right)^d.
\end{equation}
Conversely, for each integer $d\ge1$, the same binary class $\mathcal H_d$ from Theorem~\ref{thm:lower} has Littlestone dimension $d$ and, for every $0<\delta<1/4$, every horizon $T\ge T_0(d,1,\delta)$ from that theorem satisfies
\begin{equation}
  \label{eq:binary-lower}
  N_{\mathrm{bin}}^{\mathrm{dyn}}(\mathcal H_d,\delta,T) >\left(\frac{1}{4\delta}\right)^d.
\end{equation}
\end{corollary}
\begin{proof}
Apply Theorem~\ref{thm:main} with $p=1$ and $\varepsilon=\delta/2$. The sequential fat-shattering dimension at scale $\delta/64$ equals $d$, so the resulting real-valued family has size at most $(2C/\delta)^d$. Round each prediction to $\bar g_t(x_{1:t})=\mathbf 1\{g_t(x_{1:t})\ge1/2\}$. For a binary target,
\begin{equation*}
  \mathbf 1\{\bar g_t(x_{1:t})\ne h(x_t)\} \le2|g_t(x_{1:t})-h(x_t)|.
\end{equation*}
The rounded family is causal and has normalized Hamming error at most $\delta$, proving \eqref{eq:binary-upper} with $C_{\mathrm{bin}}=2C$. For the lower bound, apply Theorem~\ref{thm:lower} with $p=1$ and $\varepsilon=\delta$ at each $T\ge T_0(d,1,\delta)$, and use $N_{\mathrm{bin}}^{\mathrm{dyn}}\ge N_1^{\mathrm{dyn}}$.
\end{proof}

In particular, Corollary~\ref{cor:binary} resolves the global approximate covering question posed by Alon, Ben-Eliezer, Dagan, Moran, Naor, and Yogev~\cite{alon2021adversarial} and the dynamic Haussler problem of Malliaris and Moran~\cite{malliaris2021agnostic}. For binary vectors, $\rho_{p,T}(u,v)^p=d_{\mathrm H,T}(u,v)$. For $0<\varepsilon<1$, taking $\delta=\varepsilon^p$ in \eqref{eq:binary-upper} therefore also gives a global cover by \emph{binary} experts with normalized $\ell_p$ accuracy $\varepsilon$ and size at most $(C_{\mathrm{bin}}/\varepsilon^p)^d$.

\paragraph{Covers on a fixed instance tree.}
For a tree $\mathbf x=(\mathbf x_t)_{t=1}^T$ and a sign path $\sigma\in\{-1,+1\}^T$, write $\mathbf x(\sigma)=(\mathbf x_t(\sigma_{<t}))_{t=1}^T$. If $\mathbf x$ is an instance tree, $f(\mathbf x(\sigma))$ denotes the vector obtained by applying $f$ to each coordinate. The normalized \emph{sequential $\ell_p$ covering number} $N_p^{\mathrm{seq}}(\mathcal F,\varepsilon,\mathbf x)$ is the least cardinality of a finite family $V$ of $[0,1]$-valued prediction trees of depth $T$ such that
\begin{equation}
  \label{eq:sequential-cover}
  \forall f\in\mathcal F\; \forall\sigma\in\{-1,+1\}^T\; \exists\mathbf v\in V:\qquad \rho_{p,T}\bigl(\mathbf v(\sigma),f(\mathbf x(\sigma))\bigr) \le\varepsilon.
\end{equation}
As before, the covering number is $\infty$ if no finite family exists. Here $V$ may depend on the full tree $\mathbf x$, as in the sequential covering definition of Rakhlin, Sridharan, and Tewari~\cite{rakhlin2015sequential}. Set
\begin{equation*}
  N_p^{\mathrm{seq}}(\mathcal F,\varepsilon,T) :=\sup_{\mathbf x} N_p^{\mathrm{seq}}(\mathcal F,\varepsilon,\mathbf x),
\end{equation*}
where the supremum ranges over all instance trees of depth $T$.

\begin{corollary}[Sequential $\ell_p$ covering]
\label{cor:tree}
For every $1\le p<\infty$, $T\ge1$, and $\varepsilon>0$,
\begin{equation}
  \label{eq:seq-versus-global}
  N_p^{\mathrm{seq}}(\mathcal F,\varepsilon,T) \le N_p^{\mathrm{dyn}}(\mathcal F,\varepsilon,T).
\end{equation}
Consequently, if $0<\varepsilon<1/2$ and $d=\operatorname{sfat}_{\varepsilon/32}(\mathcal F)<\infty$, then
\begin{equation}
  \label{eq:sequential-upper}
  N_p^{\mathrm{seq}}(\mathcal F,\varepsilon,T) \le\left(\frac{C}{\varepsilon^p}\right)^d.
\end{equation}
For each $d\ge1$, the same class $\mathcal H_d$ from Theorem~\ref{thm:lower} witnesses the sharp accuracy exponent $pd$ for these sequential covering numbers, for all finite $p$, all $0<\varepsilon<2^{-1-1/p}$, and every $T\ge T_0(d,p,\varepsilon)$ from that theorem.
\end{corollary}
\begin{proof}
Fix an instance tree $\mathbf x$ and a global family $\mathcal G$. For each $g\in\mathcal G$, define a prediction tree by restriction:
\begin{equation*}
  \mathbf v_t^g(\sigma_{<t}) :=g_t\bigl(\mathbf x_1(\varnothing), \mathbf x_2(\sigma_1),\ldots, \mathbf x_t(\sigma_{<t})\bigr).
\end{equation*}
For every target $f$ and path $\sigma$, the global covering property applied to the instance sequence $\mathbf x(\sigma)$ supplies a member of this family satisfying \eqref{eq:sequential-cover}. Its cardinality is at most $|\mathcal G|$. Taking the minimum over global covers and then the supremum over instance trees proves \eqref{eq:seq-versus-global}; if no finite global cover exists, the inequality is immediate. Theorem~\ref{thm:main} gives \eqref{eq:sequential-upper}.

For sharpness, fix any $T\ge T_0(d,p,\varepsilon)$ from Theorem~\ref{thm:lower}, take the length-$T$ prefix $\bar x_{1:T}$ of its periodic sequence, and set $\mathbf x_t(\sigma_{<t})=\bar x_t$ at every node of level $t$. Fixing any one sign path turns each prediction tree into one covering vector for that prefix. The fixed-sequence lower bound therefore gives
\begin{equation*}
  N_p^{\mathrm{seq}}(\mathcal H_d,\varepsilon,T) >\left(\frac{1}{2^{p+1}\varepsilon^p}\right)^d
\end{equation*}
for the fixed class in that theorem, throughout its stated parameter range and at every $T\ge T_0(d,p,\varepsilon)$.
\end{proof}

In particular, the case $p=2$ resolves the horizon-independent sequential $\ell_2$ covering question posed by Rakhlin, Sridharan, and Tewari~\cite{rakhlin2015sequential}.

\section{Correction Schedules and Adaptive Sampling}
\label{sec:foundations}

We establish the deterministic and probabilistic ingredients used to prove Proposition~\ref{prop:core}. We first justify the reduction to a separated class and construct finite families of rank-decreasing correction schedules. We then encode prediction errors by partial binary functions and prove a sampling estimate that is uniform over such functions, even when the inputs depend on past sampling choices. The choice of sampling parameters and the construction of the master and replay histories are deferred to Section~\ref{sec:accuracy}.

\subsection{Separation and Rank-Decreasing Centers}
\label{subsec:foundations-separation}

The reduction in Section~\ref{subsec:overview-reduction} uses a uniformly separated subclass, chosen independently of the instance sequence. We first give the existence argument, including the case of an infinite function class.

\begin{lemma}[Uniformly separated reduction]
\label{lem:net}
For every nonempty $\mathcal F\subseteq[0,1]^{\mathcal X}$ and every $\eta>0$, there is a nonempty subclass $\mathcal H\subseteq\mathcal F$ such that
\begin{equation}
  \label{eq:separated-net}
  \|h-h'\|_\infty>\eta\quad(h\ne h',\ h,h'\in\mathcal H), \qquad \forall f\in\mathcal F\;\exists h\in\mathcal H: \|f-h\|_\infty\le\eta.
\end{equation}
Moreover, $\operatorname{sfat}_{\eta}(\mathcal H) \le\operatorname{sfat}_{\eta}(\mathcal F)$.
\end{lemma}
\begin{proof}
Order the uniformly greater-than-$\eta$ separated subclasses of $\mathcal F$ by inclusion. The union of a chain is still separated: any two members of the union belong together to one class in the chain. Zorn's lemma therefore supplies a maximal separated subclass $\mathcal H$. It is nonempty, since a singleton could otherwise be added. If some $f\in\mathcal F$ had distance greater than $\eta$ from every member of $\mathcal H$, adjoining $f$ would contradict maximality. This proves the approximation property. The dimension inequality follows from monotonicity under inclusion.
\end{proof}

Together with the calculation in \eqref{eq:overview-lp-reduction}, Lemma~\ref{lem:net} reduces the main upper bound to Proposition~\ref{prop:core}. For the proof of that proposition, fix a nonempty separated class $\mathcal H$ and parameters $\eta,R,T,d$ satisfying its hypotheses. The \emph{rank} of a class $V$ means $\operatorname{sfat}_{\eta}(V)$. Separation is needed for the following consequence of rank zero, rather than for the schedule construction itself.

\begin{lemma}[Rank-zero classes]
\label{lem:rank-zero}
If $\varnothing\ne V\subseteq[0,1]^{\mathcal X}$ has rank zero, then
\begin{equation}
  \label{eq:rank-zero-diameter}
  \sup_{f\in V}f(x)-\inf_{f\in V}f(x)\le\eta \qquad(x\in\mathcal X).
\end{equation}
If, in addition, $V$ is a subclass of a uniformly greater-than-$\eta$ separated class, then $V$ is a singleton.
\end{lemma}
\begin{proof}
Two functions whose values at an instance differ by at least $\eta$ would shatter a depth-one tree, using the midpoint of their values as the threshold. Thus every pairwise difference at a fixed instance is less than $\eta$. Taking the supremum of these differences gives \eqref{eq:rank-zero-diameter}; the supremum need not be attained. If distinct $h,h'\in V$ had $\|h-h'\|_\infty>\eta$, some instance would satisfy $|h(x)-h'(x)|>\eta$, contradicting the same argument.
\end{proof}

In particular, a separated class of rank zero is covered exactly by its unique member. This proves the $d=0$ case of Proposition~\ref{prop:core}. The remaining local lemmas apply to arbitrary finite-rank classes, so they may also be used for version spaces arising later in the construction.

\paragraph{Rank-decreasing centers and sign updates.}

We need a prediction rule determined by the current version space and instance, with the property that a sufficiently separated correction strictly lowers rank. The elementary operation behind this rule is to attach two shattered trees below a common root.

\begin{lemma}[Gluing shattered trees]
\label{lem:glue}
Let $V\subseteq[0,1]^{\mathcal X}$, $x\in\mathcal X$, and $z_1-z_0\ge\eta$. If both classes
\begin{equation*}
  \{f\in V:f(x)\le z_0\}, \qquad \{f\in V:f(x)\ge z_1\}
\end{equation*}
have rank at least an integer $r\ge0$, then $\operatorname{sfat}_{\eta}(V)\ge r+1$.
\end{lemma}
\begin{proof}
Attach depth-$r$ shattered trees of the two subclasses below a new root at $x$ with threshold $(z_0+z_1)/2$; the assumption $z_1-z_0\ge\eta$ guarantees the required margin $\eta/2$ at the root.
\end{proof}

For a nonempty class $V$ of finite rank $r=\operatorname{sfat}_{\eta}(V)$, define the two critical thresholds and their center by
\begin{equation}
  \label{eq:rank-center}
  \begin{aligned}
    \ell(V,x) &:=\inf\bigl\{z\in[0,1]: \operatorname{sfat}_{\eta}(\{f\in V:f(x)\le z\})=r\bigr\},\\
    u(V,x) &:=\sup\bigl\{z\in[0,1]: \operatorname{sfat}_{\eta}(\{f\in V:f(x)\ge z\})=r\bigr\},\\
    a(V,x)&:=\tfrac12\bigl(\ell(V,x)+u(V,x)\bigr).
  \end{aligned}
\end{equation}
Both defining sets are nonempty: the lower fiber equals $V$ at $z=1$, and the upper fiber equals $V$ at $z=0$. Hence the thresholds and the center belong to $[0,1]$. We do not assume that the infimum or supremum is attained.

\begin{lemma}[Rank-decreasing center]
\label{lem:center}
For every such $V$ and $x$, the two sign-update classes
\begin{equation}
  \label{eq:sign-update-fibers}
  V_x^-:=\{f\in V:f(x)\le a(V,x)-\eta\}, \qquad V_x^+:=\{f\in V:f(x)\ge a(V,x)+\eta\}
\end{equation}
have rank at most $r-1$, with the convention $\operatorname{sfat}_{\eta}(\varnothing)=-1$.
\end{lemma}
\begin{proof}
Abbreviate the thresholds by $\ell,u$ and the center by $a$. Suppose first that $u-\ell>\eta$. Choose $\ell<z_0<z_1<u$ with $z_1-z_0\ge\eta$. By the definition of $\ell$, some admissible lower threshold is smaller than $z_0$; monotonicity therefore makes the lower fiber at $z_0$ have rank $r$. The corresponding argument for $u$ makes the upper fiber at $z_1$ have rank $r$. Lemma~\ref{lem:glue} would give rank at least $r+1$ for $V$, a contradiction. Consequently $u-\ell\le\eta$, and
\begin{equation*}
  a-\eta\le\ell-\eta/2<\ell, \qquad a+\eta\ge u+\eta/2>u.
\end{equation*}
If $a-\eta<0$, the lower sign-update class is empty. Otherwise, its threshold is strictly below $\ell$, so its rank cannot be $r$. Similarly, the upper class is either empty because $a+\eta>1$, or has rank smaller than $r$ because its threshold exceeds $u$. Ranks are integers, which proves both bounds.
\end{proof}

A \emph{sign update} replaces the current nonempty version space $V$ by $V_x^-$ or $V_x^+$. A nonempty update is called valid. Starting from a class of rank at most $r$, any sequence of $i$ valid updates leaves a version space of rank at most $r-i$; in particular, there can be at most $r$ valid updates. The update is specified by its sign and uses the fixed center rule in \eqref{eq:rank-center}, not a target value or a discretized approximation to that value.

\subsection{Correction Schedules and Pointwise Covers}
\label{subsec:foundations-schedules}

We next turn prescribed sign updates into label-free causal simulations. Fix an initial nonempty version space $V_0$ of rank at most $r$ and a positive integer $m$. A selection history specifies which rounds are selected. A selected round has an \emph{ordinal} $s$, meaning that it is the $s$th selected round, regardless of its position in the original instance sequence.

A \emph{correction schedule} of depth $i$ is a list
\begin{equation}
  \label{eq:sign-schedule}
  J=((s_1,\sigma_1),\ldots,(s_i,\sigma_i)), \qquad 1\le s_1\le\cdots\le s_i\le m, \qquad \sigma_j\in\{-,+\}.
\end{equation}

The empty list is the depth-zero schedule. Repeated ordinals are allowed because more than one rank-decreasing correction may be needed at a single instance. The classical Sauer--Shelah--Perles bound~\cite{sauer1972density,shelah1972combinatorial} has dynamic analogues based on the mistake-driven Standard Optimal Algorithm~\cite{ben2009agnostic,malliaris2021agnostic,alon2021adversarial}. Scale-sensitive SOA-based expert constructions for real-valued prediction were developed in~\cite{rakhlin2010online}, encoding correction rounds and discretized prediction values. In contrast, our schedules record only selected ordinals and signs, allowing multiple rank-decreasing updates at one instance without encoding target values.

\paragraph{Simulation and prediction timing.}
Every simulation starts at $V_0$. On round $t$, after observing $x_t$ but before the current selection is revealed, its \emph{provisional prediction} is $a(V,x_t)$, using the version space left by the preceding round. If this round has selected ordinal $s$, the simulation executes all entries of $J$ with $s_j=s$, in their listed order. Each entry applies the indicated sign update, and the center is recomputed from the new version space before the next entry is executed. The simulation emits its \emph{final prediction} only after all these updates. On an unselected round it makes no update, and its final prediction equals its provisional prediction. We denote the provisional and final predictions by $a_t^J$ and $y_t^J$, respectively.

An update to an empty fiber makes the simulation invalid. From that update onward, it remains in an absorbing invalid state and emits zero, including on the current round. If the selection history ends before a prescribed ordinal is reached, the corresponding update is not executed. The same rule is also defined on selection histories longer than $m$: there are simply no prescribed updates at ordinals larger than $m$. These conventions ensure that every formal schedule has an output on every outcome, not only on outcomes where the intended sampling cap is respected.

The simulation uses only the instance history, the selection history, and its fixed schedule. If the selection policy is a deterministic causal function of the instances, the simulation defines a deterministic causal expert. Under random selection, its outputs additionally depend on the selection history, while remaining causal. We will first analyze these simulations and later obtain deterministic experts by the argument in Section~\ref{sec:determinization}.

A valid schedule is \emph{born} once all its prescribed corrections have occurred. The empty schedule is born at time zero. After birth, its version space remains fixed, although its predictions may still vary with the current instance. A born depth-$i$ schedule has residual rank at most $r-i$. Moreover, let $q_t$ be the number of selected rounds up to round $t$. Removing from a schedule $J$ all entries with selection ordinals greater than $q_t$ leaves its version spaces and predictions unchanged on rounds $1,\ldots,t$. The resulting prefix schedule has completed all its prescribed corrections by round $t$, provided they are valid.

\paragraph{The formal schedule family.}
The number of weakly increasing ordinal lists of length $i$ in $[m]$ is $\binom{m+i-1}{i}$, and every list admits $2^i$ sign assignments. Thus, writing $n_0(m)=1$, the exact counts are
\begin{equation}
  \label{eq:Q}
  n_i(m):=2^i\binom{m+i-1}{i}, \qquad Q_r(m):=\sum_{i=0}^r n_i(m).
\end{equation}
These counts include invalid schedules, schedules that never become born, and schedules whose outputs coincide. Consequently $Q_r(m)$ is an upper bound on the number of distinct prediction rules, with no need to recognize or remove such duplications.

\begin{lemma}[Pointwise schedule cover]
\label{lem:pointwise}
Let $\varnothing\ne V_0\subseteq[0,1]^{\mathcal X}$ satisfy $\operatorname{sfat}_{\eta}(V_0)\le r$. With every round selected, the schedules of depth at most $r$ form one family of at most $Q_r(m)$ deterministic causal experts such that, on every instance sequence of length at most $m$, each target in $V_0$ is approximated by one expert with pointwise error at most $\eta$. More generally, on any selection history with at most $m$ selected rounds, a schedule in the same formal family attains error at most $\eta$ on every selected round.
\end{lemma}
\begin{proof}
Fix a target $f\in V_0$ and an instance-and-selection history. Consider the following reference procedure, which is used only to find a witnessing schedule. It starts from $V_0$ and does not update on unselected rounds. At a selected instance $x$, while $|f(x)-a(V,x)|>\eta$, it updates toward $f(x)$: it uses the plus fiber if $f(x)>a(V,x)+\eta$ and the minus fiber otherwise. It recomputes the center and repeats at the same instance until the error is at most $\eta$.

Every update retains $f$, and Lemma~\ref{lem:center} decreases the rank by at least one. Since the version space stays nonempty, at most $r$ updates occur over the entire history. Thus, the correction loop at each selected round terminates with prediction error at most $\eta$. Record the selected ordinal and sign of each update, retaining their order when an ordinal occurs repeatedly. The resulting list is a schedule of depth at most $r$ in \eqref{eq:sign-schedule}. Its ordinary simulation has exactly the reference procedure's states and final predictions, by induction on rounds and on updates within a round. It therefore has error at most $\eta$ at every selected round.

When all rounds are selected, this is the asserted pointwise cover. The formal family is fixed independently of both the target and the instance sequence. Only the choice of its witnessing member used the target values; no expert in the family receives those values.
\end{proof}

We record both the size estimate and a convolution identity for these counts. The latter will be used when counting the continuation capacity in Section~\ref{sec:determinization}.

\begin{lemma}[Schedule-count estimates]
\label{lem:count}
For integers $m,r\ge1$,
\begin{equation}
  \label{eq:Q-bound}
  Q_r(m)\le2^r\binom{m+r}{r} \le\left[2e\left(1+\frac mr\right)\right]^r.
\end{equation}
For every integer $\ell\ge0$,
\begin{equation}
  \label{eq:convolution}
  \sum_{i=0}^{\ell}n_i(m)n_{\ell-i}(m)=n_{\ell}(2m).
\end{equation}
Consequently, for every integer $d\ge1$,
\begin{equation}
  \label{eq:triangular-convolution}
  \sum_{i=0}^{d-1}n_i(m)Q_{d-i}(m)\le Q_d(2m).
\end{equation}
\end{lemma}
\begin{proof}
The hockey-stick identity gives
\begin{equation*}
  Q_r(m) \le2^r\sum_{i=0}^r\binom{m+i-1}{i} =2^r\binom{m+r}{r}.
\end{equation*}
Using $\binom Nr\le(eN/r)^r$ proves \eqref{eq:Q-bound}. For the convolution, each of the $m$ selection ordinals can carry any number $j\ge0$ of corrections, with $2^j$ possible sign assignments. Thus each ordinal contributes the generating function
\begin{equation*}
  \sum_{j=0}^{\infty}2^jz^j=(1-2z)^{-1},
\end{equation*}
and multiplying over all $m$ ordinals gives
\begin{equation*}
  \sum_{i=0}^{\infty}n_i(m)z^i=(1-2z)^{-m}.
\end{equation*}
Squaring this identity and comparing the coefficient of $z^\ell$ gives \eqref{eq:convolution}. Summing it over $\ell\le d$ yields
\begin{equation*}
  \sum_{\substack{i,j\ge0\\i+j\le d}}n_i(m)n_j(m)=Q_d(2m).
\end{equation*}
The sum in \eqref{eq:triangular-convolution} contains all of these nonnegative terms except the one indexed by $(i,j)=(d,0)$.
\end{proof}

In particular, the all-round cover has size $Q_d(T)\le[2e(1+T/d)]^d$ when $d\ge1$. This will handle the short-horizon regime directly. For larger horizons, the same schedule family will be used on shorter sampled lists. In either case, the count depends on selected ordinals and signs, without an additional factor for discretizing the target values at scale $\eta$.

\subsection{Partial Binary Completion and Adaptive Sampling}
\label{subsec:foundations-partial}

The sampling argument encodes one-sided prediction errors using partial binary labels: discrepancies at most $\eta$ receive label $0$, those at least $2\eta$ receive label $1$, and intermediate values receive the undefined symbol $*$. This gap lets us control the resulting partial binary classes by sequential fat-shattering dimension at scale $\eta$. We first define these classes, then construct the finite families of binary rules needed to analyze their sampled values.

Let $\mathcal Z$ be a domain. A \emph{partial binary class} is a subset $\mathcal A\subseteq\{0,1,*\}^{\mathcal Z}$. Such a class shatters a binary instance tree $z_t:\{0,1\}^{t-1}\to\mathcal Z$, $t\in[s]$, if
\begin{equation*}
  \forall b\in\{0,1\}^s\;\exists\phi\in\mathcal A\; \forall t\in[s]:\quad \phi(z_t(b_{<t}))=b_t.
\end{equation*}
Thus $*$ realizes neither edge of the tree. Define $\operatorname{Ldim}(\mathcal A)$ to be the largest shattered depth, or infinity when these depths are unbounded. As before, a nonempty class has dimension at least zero, and the empty class has dimension $-1$. On total binary classes this is ordinary Littlestone dimension.

\paragraph{Partial error classes.}
For $f\in\mathcal H$ and $(x,a)\in\mathcal X\times[0,1]$, define
\begin{equation}
  \label{eq:partial-errors}
  \begin{aligned}
    \phi_f^+(x,a)&:=
    \begin{cases}
      0,&f(x)-a\le\eta,\\
      1,&f(x)-a\ge2\eta,\\
      *,&\eta<f(x)-a<2\eta,
    \end{cases}
    & \phi_f^-(x,a)&:=
    \begin{cases}
      0,&a-f(x)\le\eta,\\
      1,&a-f(x)\ge2\eta,\\
      *,&\eta<a-f(x)<2\eta.
    \end{cases}
  \end{aligned}
\end{equation}
Write $\mathcal A^\pm:=\{\phi_f^\pm:f\in\mathcal H\}$. A definite one records a one-sided error of size at least $2\eta$; a nonzero value, including $*$, records an error strictly greater than $\eta$ in that direction.

\begin{lemma}[Dimension of the partial error classes]
\label{lem:partial-errors}
Both partial classes satisfy
\begin{equation}
  \label{eq:partial-error-dimension}
  \operatorname{Ldim}(\mathcal A^+),\quad \operatorname{Ldim}(\mathcal A^-) \le\operatorname{sfat}_{\eta}(\mathcal H).
\end{equation}
\end{lemma}
\begin{proof}
Suppose a depth-$s$ tree of pairs $(x,a)$ is shattered by $\mathcal A^+$. Replacing each node $(x,a)$ with instance $x$ and threshold $a+3\eta/2$ yields an $\eta$-fat-shattered tree, since labels $0$ and $1$ imply $f(x)\le a+\eta$ and $f(x)\ge a+2\eta$, respectively. The argument for $\mathcal A^-$ is identical, using threshold $a-3\eta/2$ and interchanging the two branches at each node.
\end{proof}

\paragraph{Completion by finitely many causal rules.}
The SOA mistake-bound argument extends to partial concept classes~\cite{alon2022partial}. Combining this principle with the mistake-pattern expert construction of~\cite{ben2009agnostic}, we obtain causal completion rules that agree with each partial target at all defined positions.
For a partial class $W$ and $b\in\{0,1\}$, write $W_z^b:=\{\phi\in W:\phi(z)=b\}$. For nonempty $W$ of finite Littlestone dimension, the standard optimal algorithm (SOA)~\cite{littlestone1988learning} prediction $\operatorname{soa}(W,z)$ is the label maximizing $\operatorname{Ldim}(W_z^b)$, with ties broken in favor of zero. As in Lemma~\ref{lem:glue}, the two defined-label fibers cannot both have the rank of $W$. Since SOA chooses a maximum-rank fiber,
\begin{equation}
  \label{eq:partial-soa-drop}
  \operatorname{Ldim} \bigl(W_z^{1-\operatorname{soa}(W,z)}\bigr) \le\operatorname{Ldim}(W)-1.
\end{equation}
Undefined values belong to neither fiber.

\begin{lemma}[Partial SOA completion]
\label{lem:completion}
Let $\mathcal A\subseteq\{0,1,*\}^{\mathcal Z}$ have Littlestone dimension at most $r\ge0$, and let $n\ge1$ be an integer. There is a family of at most
\begin{equation*}
  S_r(n):=\sum_{i=0}^{\min\{r,n\}}\binom ni
\end{equation*}
deterministic causal prediction rules $v=(v_t)_{t=1}^n$, with $v_t:\mathcal Z^t\to\{0,1\}$, such that, on every sequence of length at most $n$, each $\phi\in\mathcal A$ is matched at all its defined positions by one rule.
\end{lemma}
\begin{proof}
If $\mathcal A$ is empty, a constant rule suffices. Otherwise, for each $J\subseteq[n]$ with $|J|\le r$, define a rule starting from $W=\mathcal A$. Outside $J$, it outputs $\operatorname{soa}(W,z)$ and makes no update. At a round in $J$, it outputs the opposite label and restricts $W$ to that opposite-label fiber. An empty restriction makes the rule invalid, with output zero from that round onward. Thus every $J$ specifies a total binary causal rule without access to partial target values. There are $S_r(n)$ such sets $J$.

To find a rule for a particular target $\phi$ and sequence, use a reference procedure that restricts its version space only when $\phi(z)$ is defined and differs from the current SOA prediction. It records each such round and restricts to $W_z^{\phi(z)}$. It leaves $W$ unchanged when the prediction is correct or when $\phi(z)=*$. Every restriction retains $\phi$ and, by \eqref{eq:partial-soa-drop}, strictly decreases rank. The recorded set $J$ therefore has size at most $r$. Induction on rounds shows that the rule indexed by this $J$ has exactly the reference procedure's states and matches every defined target value.
\end{proof}

The lemma constructs causal completions along input histories; it does not choose one total static completion of each partial target. Its codes record at most one binary reversal per round and require no sign, which explains the count $S_r(n)$ rather than $Q_r(n)$. We will also use the elementary estimate
\begin{equation}
  \label{eq:S-bound}
  S_r(n)\le\left(\frac{en}{r}\right)^r \qquad(1\le r\le n).
\end{equation}
Indeed, with $x=r/n\le1$,
\begin{equation*}
  x^r S_r(n) \le\sum_{i=0}^r\binom ni x^i \le(1+x)^n\le e^{nx}=e^r.
\end{equation*}
For $r=0$ the count is one, and for $r\ge n$ it is $2^n$.

\paragraph{Adaptive low-hit sampling.}
Sampling against inputs that react to earlier sample selections has been studied in~\cite{beneliezer2020sampling,alon2021adversarial}. We adapt the double-sampling argument of~\cite{alon2021adversarial}
to partial binary classes. Causal completion rules handle undefined labels, while a red--blue martingale bounds the probability of many definite ones with few sampled nonzero values. The witnessing completion rule matches all defined labels and follows the current SOA prediction at undefined positions, without updating its version space. Consequently, the number of sampled ones it predicts is at most the number of sampled nonzero values of the partial target.

\paragraph{Order of information.}
Let $\xi_1,\ldots,\xi_T$ be independent Bernoulli-$\gamma$ bits. The input $z_t$ is chosen before $\xi_t$ is revealed and may depend on $\xi_{<t}$ and an independent finite-support auxiliary seed $\omega$. Equivalently, $z_t$ is a deterministic function of $(\xi_{<t},\omega)$. We call this a \emph{predictable input stream} for the Bernoulli sample. Fixing $\omega$ leaves the product law of the sampling bits unchanged.

\begin{lemma}[Adaptive low-hit sampling]
\label{lem:lowhit}
Let $\mathcal A\subseteq\{0,1,*\}^{\mathcal Z}$ have Littlestone dimension at most an integer $r\ge0$. Let $z_1,\ldots,z_T$ be a predictable input stream for independent Bernoulli-$\gamma$ bits $\xi_1,\ldots,\xi_T$, where $0<\gamma\le1/2$. For each $\phi\in\mathcal A$, define
\begin{equation}
  \label{eq:lowhit-counts}
  L_\phi:=\sum_{t=1}^T\mathbf 1\{\phi(z_t)=1\}, \qquad H_\phi:=\sum_{t=1}^T\xi_t\mathbf 1\{\phi(z_t)\ne0\}.
\end{equation}
Let $q\ge0$ and $B\ge2q$ satisfy $\gamma B\ge16\log2$. With $M=\lceil4\gamma T\rceil$ and $c_0=\log(3/2)/4$, one has
\begin{equation}
  \label{eq:lowhit}
  \Pr\{\exists\phi\in\mathcal A: L_\phi\ge B,\ H_\phi\le q\} \le2\left(S_r(M)2^q e^{-c_0\gamma B} +e^{-2\gamma T/3}\right).
\end{equation}
\end{lemma}
\begin{proof}
Fix a value of the auxiliary seed $\omega$ in its support. The sampling bits remain independent Bernoulli-$\gamma$ variables, and each $z_t$ is now a deterministic function of $\xi_{<t}$. All probabilities below are conditional on this fixed seed until the final averaging step. Call a round \emph{red} when $\xi_t=1$, and let $E$ denote the event on the left of \eqref{eq:lowhit}. For this fixed auxiliary seed, $E$ is measurable, as it is determined by the finite-valued vector $\xi_{1:T}$.

\paragraph{A second sample for comparison.}
We introduce a second sample to detect definite ones outside the red sample. After the red bits have been generated, independently color each nonred round \emph{blue} with probability $\gamma/(1-\gamma)$. This choice makes the marginal probability of blue equal to $\gamma$, the same as that of red.

Fix a red-bit outcome in $E$, and let $\phi\in\mathcal A$ be any witness. Conditional on this outcome, both the input sequence and $\phi$ are fixed. The witness has at least $B-q\ge B/2$ nonred definite-one positions, since its red definite ones are included among its at most $q$ red nonzero values. Its number of blue definite ones is therefore binomial with mean
\begin{equation*}
  \nu\ge\frac{\gamma}{1-\gamma}(B-q)\ge\gamma B/2.
\end{equation*}
Since $\gamma B/4\le\nu/2$, a standard Chernoff bound gives, conditional on the red outcome,
\begin{equation*}
  \Pr\{\text{at least }\gamma B/4\text{ blue definite ones}\} \ge1-e^{-\nu/8} \ge1-e^{-\gamma B/16}\ge\tfrac12.
\end{equation*}
Let $E'$ be the event that some partial target has at most $q$ red nonzero values and at least $\gamma B/4$ blue definite ones. Thus $\Pr(E'\mid\xi_{1:T})\ge1/2$ on $E$, and hence
\begin{equation}
  \label{eq:ghost-reduction}
  \Pr(E)\le2\Pr(E').
\end{equation}

\paragraph{Conditioning on the sampled union.}
The three possible colors---red, blue, and neither---have probabilities $\gamma,\gamma,1-2\gamma$, independently across rounds. Equivalently, we may first generate independent union bits $U_t\sim\operatorname{Bernoulli}(2\gamma)$ and then give each included round an independent fair red-or-blue color. Setting $\xi_t=1$ exactly on red rounds and reconstructing $z_t$ from $\xi_{<t}$ gives the original joint distribution of the red bits and inputs.

Write $N=\sum_{t=1}^T U_t$. Since $\mathbb E N=2\gamma T$ and $M\ge4\gamma T$, another standard Chernoff bound gives
\begin{equation}
  \label{eq:ghost-union-size}
  \Pr\{N>M\}\le e^{-2\gamma T/3}.
\end{equation}
Fix a union pattern with $N\le M$, and enumerate its included rounds as $\tau_1<\cdots<\tau_N$. Conditional on this pattern, the colors at these rounds remain independent and fair. Since $z_t$ depends only on the preceding sampling bits $\xi_{<t}$, fixing the union pattern makes $z_{\tau_j}$ a function of the colors at $\tau_1,\ldots,\tau_{j-1}$. Indeed, $\xi_t=1$ exactly when round $t$ is included and colored red. Thus the compressed stream $z_{\tau_1},\ldots,z_{\tau_N}$ is predictable with respect to the independent fair colors. Conditioning on the union has not fixed this input stream; it may still depend on the preceding colors.

Use the fixed family $\mathcal V$ of at most $S_r(M)$ deterministic binary causal rules supplied by Lemma~\ref{lem:completion} at horizon $M$. On $E'$, some rule in $\mathcal V$ matches all defined labels of a witnessing target on the compressed stream. It has at least $\gamma B/4$ blue ones, because it matches every blue definite one. It has at most $q$ red ones, because it must output zero at every red position where the target is zero; all other red positions, including undefined ones, are included in the target's red nonzero count.

\paragraph{Bounding the red--blue discrepancy.}
Still conditioning on the union pattern, fix a rule $v\in\mathcal V$. Let $A_j$ and $C_j$ count its red and blue ones, respectively, among the first $j$ included rounds. The rule's prediction at the next included round is determined before that round's color is revealed. Therefore
\begin{equation*}
  Z_j:=2^{-A_j}(3/2)^{C_j},\qquad 0\le j\le N,
\end{equation*}
is a nonnegative martingale with $Z_0=1$. When the prediction is zero, its multiplier is one; when the prediction is one, the conditional expected multiplier is
\begin{equation*}
  \tfrac12\cdot\tfrac12+\tfrac12\cdot\tfrac32=1.
\end{equation*}
On $\{A_N\le q,\ C_N\ge\gamma B/4\}$, one has $Z_N\ge2^{-q}(3/2)^{\gamma B/4}$. Markov's inequality gives
\begin{equation*}
  \Pr\{A_N\le q,\ C_N\ge\gamma B/4\mid U_{1:T}\} \le2^q(3/2)^{-\gamma B/4} =2^q e^{-c_0\gamma B}.
\end{equation*}
Take a union bound over $\mathcal V$, average over union patterns with $N\le M$, and add \eqref{eq:ghost-union-size}. We obtain
\begin{equation*}
  \Pr(E')\le S_r(M)2^q e^{-c_0\gamma B}+e^{-2\gamma T/3}.
\end{equation*}
Combining this with \eqref{eq:ghost-reduction} proves the asserted bound conditional on each value of $\omega$. Averaging over $\omega$ completes the proof.
\end{proof}

For the application to prediction errors, fix a formal schedule $J$ and set $z_t=(x_t,a_t^J)$, where $a_t^J$ is its provisional prediction. Under the stated information order, this stream is predictable. For either partial error class in \eqref{eq:partial-errors}, a sampled nonzero value records a one-sided provisional error exceeding $\eta$. When $J$ is the schedule produced for a target $h$ by the reference procedure in Lemma~\ref{lem:pointwise}, each such round requires at least one correction, so $H_{\phi_h^\pm}\le d$.

Section~\ref{sec:accuracy} will apply Lemma~\ref{lem:lowhit} to each fixed schedule and take a union bound over all schedules, so the resulting guarantee also holds for the witnessing schedule of every target.

\section{Master Histories and Multi-Pool Accuracy}
\label{sec:accuracy}

The preceding section gives sign schedules that are accurate on selected rounds. We now introduce multiple sampling pools to obtain a simultaneous guarantee: for every target, a fixed positive fraction of the pools provide schedules with few large prediction errors over the entire horizon, including unselected rounds. We will also ensure that the target satisfies every constraint recorded by these schedules with a positive margin.

\paragraph{Parameters and capped sampling pools.}

Fix $\mathcal H,\eta,R,T,d$ as in Proposition~\ref{prop:core}. Recall that the rank of a version space is its sequential fat-shattering dimension at scale $\eta$. The case $d=0$ was settled in Section~\ref{subsec:foundations-separation}, so assume $d\ge1$ and put $b=T/R$. Fix a universal constant $\alpha\ge128$; we will specify sufficient conditions on its size at the end of this section.

If $b\le32\alpha d$, select every round. Lemmas~\ref{lem:pointwise} and~\ref{lem:count} give a global pointwise radius-$\eta$ cover of size
\begin{equation}
  \label{eq:accuracy-short-horizon}
  Q_d(T)\le\left[2e\left(1+\frac{T}{d}\right)\right]^d \le(66e\alpha R)^d.
\end{equation}
Its witnessing experts have no errors of size at least $8\eta$. This proves Proposition~\ref{prop:core} in this regime, after increasing its universal constant if necessary.

For the rest of the section assume $b>32\alpha d$ and set
\begin{equation}
  \label{eq:accuracy-parameters}
  k=4\lceil\log_2 R\rceil,\qquad \beta=\frac{\alpha d}{b}<\frac1{32},\qquad m=\lceil2\beta T\rceil=\lceil2\alpha dR\rceil.
\end{equation}
In particular, $k\ge24$ and $k$ is divisible by four. All logarithms below are natural unless a base is indicated.

For each $t\in[T]$ and $j\in[k]$, let $\xi_{t,j}\sim\operatorname{Bernoulli}(\beta)$, independently over all rounds and pools. The \emph{raw} selection list of pool $j$ contains the rounds with $\xi_{t,j}=1$. Its \emph{capped} list retains only the first $m$ such rounds, with indicators
\begin{equation*}
  \widehat\xi_{t,j} :=\xi_{t,j}\mathbf 1\!\left\{\sum_{s<t}\xi_{s,j}<m\right\}.
\end{equation*}
The capped list therefore has length at most $m$ on every outcome. The retained indicators are not themselves independent Bernoulli variables; all independence arguments below use the raw bits.

Define
\begin{equation}
  \label{eq:accuracy-cap-event}
  \mathsf{Cap}:= \left\{\sum_{t=1}^T\xi_{t,j}\le m \text{ for every }j\in[k]\right\}.
\end{equation}
Each raw list has expected length $\beta T=\alpha dR$. Since $m\ge2\beta T$, a Chernoff bound and a union bound give
\begin{equation}
  \label{eq:accuracy-cap-tail}
  \Pr(\mathsf{Cap}^{\mathsf c}) \le k e^{-\alpha dR/3}.
\end{equation}
On $\mathsf{Cap}$, the raw and capped selection indicators agree at every round. Thus, running the same schedule with raw or capped selections on the same realized input history produces identical states and predictions. We prove the sampling bounds under the original product law of the raw bits, without conditioning on $\mathsf{Cap}$, and use this agreement to transfer the guarantees to the capped simulations on $\mathsf{Cap}$.

\paragraph{Adaptive inputs.}
As in Lemma~\ref{lem:lowhit}, we allow adaptive input sequences. Write $\boldsymbol\xi_t=(\xi_{t,1},\ldots,\xi_{t,k})$ and $\boldsymbol\xi_{<t}=(\boldsymbol\xi_1,\ldots, \boldsymbol\xi_{t-1})$. More precisely, we allow
\begin{equation}
  \label{eq:accuracy-input-rule}
  x_t=F_t(\boldsymbol\xi_{<t},\omega_0),
\end{equation}
where $F_t$ are arbitrary deterministic maps and $\omega_0$ is an independent finite-support auxiliary random seed.

Conditional on $\omega_0$ and $\boldsymbol\xi_{<t}$, the instance $x_t$ and the provisional predictions of all fixed schedule simulations are determined, while the coordinates of $\boldsymbol\xi_t$ remain independent Bernoulli-$\beta$ variables.

All probability bounds in this section hold uniformly over input rules of the form \eqref{eq:accuracy-input-rule}.

\subsection{The Analytical Master and Uniform Accuracy}
\label{subsec:accuracy-master}

We first specify a common selection process for the reference histories, containing every round selected by any pool. The choice depends on $k\beta=\alpha kdR/T$. When $k\beta\le1/8$, we use the union of the pool selections, whose sampling rate is comparable to $k\beta$. When $k\beta>1/8$, we have $T<8\alpha kdR$, so selecting every round is affordable. Accordingly, define
\begin{equation}
  \label{eq:accuracy-master-selection}
  \zeta_t:=
  \begin{cases}
    \mathbf 1\{\boldsymbol\xi_t\ne\boldsymbol 0\}, &k\beta\le1/8,\\
    1,&k\beta>1/8.
  \end{cases}
\end{equation}
We call these cases the \emph{union regime} and the \emph{all-round regime}, respectively. In either case, $\xi_{t,j}=1$ implies $\zeta_t=1$.

In the union regime, the $\zeta_t$ are independent Bernoulli variables with rate
\begin{equation}
  \label{eq:accuracy-union-rate}
  u=1-(1-\beta)^k,\qquad \frac{k\beta}{2}\le u\le k\beta\le\frac18.
\end{equation}
The upper bound follows from a union bound, while the lower bound follows from $(1-\beta)^k\le e^{-k\beta}$ and $1-e^{-z}\ge z/2$ for $0\le z\le1$.

\paragraph{The analytical master.}
For each target $h\in\mathcal H$, define its
\emph{analytical master} by the following target-dependent
procedure. Initialize $V=\mathcal H$. At round $t$, before
revealing $\boldsymbol\xi_t$, form the provisional prediction
$a_t^h=a(V,x_t)$. If $\zeta_t=0$, leave $V$ unchanged and
set $y_t^h=a_t^h$. If $\zeta_t=1$, while
$|a(V,x_t)-h(x_t)|>\eta$, apply the sign update toward
$h(x_t)$ as in Lemma~\ref{lem:pointwise}, recomputing the
center after each update. Set $y_t^h=a(V,x_t)$ after
the correction loop.

Every update retains $h$ and decreases the sequential fat rank. Hence the procedure makes at most $d$ corrections over the entire horizon, and
\begin{equation}
  \label{eq:accuracy-master-selected}
  \zeta_t=1\quad\Longrightarrow\quad |y_t^h-h(x_t)|\le\eta.
\end{equation}
The procedure uses target values only to find a witnessing schedule; it is not a label-free prediction rule. Record the selected ordinal and sign of each correction, and denote the resulting code by $J^h$. This code is defined on every raw outcome and always has depth at most $d$, although its ordinals need not satisfy a fixed cap.

\paragraph{A fixed family of master codes.}
Set
\begin{equation}
  \label{eq:accuracy-master-code-budget}
  M_*:=4km,
\end{equation}
and let $\mathcal J_*$ be the family of all formal sign schedules of depth at most $d$ on ordinals $1,\ldots,M_*$. It is fixed before the experiment, and $|\mathcal J_*|=Q_d(M_*)$. In the union regime, $\mathsf{Cap}$ bounds the number of master selections by $km$. In the all-round regime,
\begin{equation*}
  k\beta>\frac18 \quad\Longrightarrow\quad T<8\alpha kdR\le4km=M_*.
\end{equation*}
Thus, on $\mathsf{Cap}$, $J^h\in\mathcal J_*$ for every target. There is no truncation of the master: $M_*$ is a bound used to specify the finite code family.

For every fixed $J\in\mathcal J_*$, simulate its prescribed updates on the master selection list, starting from $\mathcal H$, and denote its provisional and final predictions by $a_t^J$ and $y_t^J$. Use the conventions of Section~\ref{subsec:foundations-schedules}: an empty-fiber update makes the simulation invalid with subsequent predictions zero, and no updates are prescribed at ordinals greater than $M_*$. Every fixed-code simulation is therefore defined on every raw outcome, including outcomes outside $\mathsf{Cap}$. All simulations use the same realized inputs and selection lists; varying $J$ does not change the input rule. On an outcome with $J^h\in\mathcal J_*$, the simulation of $J^h$ has exactly the target-driven master's states and predictions, by induction on rounds and on corrections within each round.

\paragraph{Accuracy before the current selection.}
The selected-round guarantee \eqref{eq:accuracy-master-selected} concerns final predictions. We also need control of the provisional predictions on all rounds. Define
\begin{equation}
  \label{eq:accuracy-master-event}
  \mathsf{Master}:= \left\{\text{for every }h\in\mathcal H,\quad \#\{t\in[T]:|a_t^h-h(x_t)|\ge2\eta\}<\frac b{32}\right\}.
\end{equation}

\begin{lemma}[Uniform master accuracy]
\label{lem:uniform-master}
In the all-round regime, $\mathsf{Master}$ holds on every outcome. In the union regime, put $M=\lceil4uT\rceil$ and let $c_0=\log(3/2)/4$ be the constant in Lemma~\ref{lem:lowhit}. Then
\begin{equation}
  \label{eq:accuracy-master-tail}
  \Pr(\mathsf{Cap}\cap\mathsf{Master}^{\mathsf c}) \le4Q_d(M_*)\left( S_d(M)2^d e^{-c_0ub/64}+e^{-2uT/3}\right).
\end{equation}
\end{lemma}
\begin{proof}
In the all-round regime, every provisional error of size at least $2\eta$ forces a correction. There are at most $d$ such rounds, and $d<b/32$ because $b>32\alpha d$. This proves the first claim.

Now consider the union regime and fix $J\in\mathcal J_*$. We will apply Lemma~\ref{lem:lowhit} to the stream $(x_t,a_t^J)$ and the Bernoulli-$u$ master selections. The lemma requires the stream to be determined by the past sampling bits and an auxiliary seed independent of those bits. Although $\zeta_t$ are independent Bernoulli-$u$ variables, the inputs may depend on the full preceding pool vectors $\boldsymbol\xi_{<t}$, which are not determined by $\zeta_{<t}$ and $\omega_0$ alone.

An equivalent way to generate the raw vectors is to draw independent Bernoulli-$u$ bits $\zeta_t$ and independent nonzero vectors $W_t$, where each $W_t$ has the Bernoulli-$\beta$ product law conditioned on being nonzero, and to set
\begin{equation*}
  \boldsymbol\xi_t=\zeta_t W_t.
\end{equation*}
The entire sequence $(W_t)_{t=1}^T$ is independent of the union bits and of $\omega_0$. This generates exactly the original joint law of the raw pool vectors. Fixing the finite-support auxiliary seed $\omega=(\omega_0,\allowbreak W_1,\allowbreak\ldots,\allowbreak W_T)$ leaves the union bits independent Bernoulli-$u$ variables. Once this auxiliary seed is fixed, each preceding pool vector is determined by the corresponding union bit through $\boldsymbol\xi_s=\zeta_s W_s$, for $s<t$. Consequently, the stream $(x_t,a_t^J)$ is determined by $\zeta_{<t}$ and $\omega$, as required by Lemma~\ref{lem:lowhit}. For each fixed $\omega$, Lemma~\ref{lem:lowhit} applies and gives the same probability bound, independently of $\omega$. Averaging this conditional bound over $\omega$ yields the corresponding unconditional bound. Since the reconstructed pool vectors have the original joint distribution, the bound holds for the original sampling process.

For either partial error class $\mathcal A^+$ or $\mathcal A^-$ from \eqref{eq:partial-errors}, apply Lemma~\ref{lem:lowhit} with
\begin{equation*}
  \gamma=u,\qquad r=q=d,\qquad B=b/64.
\end{equation*}
The dimension hypothesis follows from Lemma~\ref{lem:partial-errors}. The numerical hypotheses follow from $u\le1/8$ and
\begin{equation*}
  \frac b{64}>\frac{\alpha d}{2}\ge2d, \qquad \frac{ub}{64}\ge\frac{\alpha kd}{128}\ge16\log2,
\end{equation*}
using $k\ge24$ and $\alpha\ge128$. For this fixed code $J\in\mathcal J_*$ and either sign, the probability that some partial target has at least $b/64$ definite ones but at most $d$ sampled nonzero values is at most
\begin{equation*}
  2\left(S_d(M)2^d e^{-c_0ub/64}+e^{-2uT/3}\right).
\end{equation*}
Take a union bound over the two signs and all $Q_d(M_*)$ fixed codes.

Outside this union of bad events, consider an outcome in $\mathsf{Cap}$ and any target $h$. Its actual code $J^h$ belongs to $\mathcal J_*$. A sampled nonzero value of either $\phi_h^+(x_t,a_t^{J^h})$ or $\phi_h^-(x_t,a_t^{J^h})$ means that the corresponding one-sided provisional error exceeds $\eta$. It therefore forces a correction in the target-driven master. For each sign, there are at most $d$ such sampled rounds, so the one-sided definite-error count is less than $b/64$. The sum of these two counts is exactly $\#\{t:|a_t^h-h(x_t)|\ge2\eta\}$, which is consequently less than $b/32$. This proves \eqref{eq:accuracy-master-tail}. The estimates were applied to fixed codes before selecting any $J^h$; no conditioning on a target's realized code was used.
\end{proof}

\subsection{Tolerant Replay and Buffered Corrections}
\label{subsec:accuracy-replay}

We next use the master's final predictions as reference values for version space updates in each individual raw pool. Each pool maintains its own version space and applies the same center rule and sign-update rule as the master, but to its own version space. Unlike the master, which corrects when its center differs from $h(x_t)$ by more than $\eta$, a pool replay corrects when its center differs from the master's final prediction by more than $4\eta$.

\paragraph{The replay rule.}
For each fixed master code $J\in\mathcal J_*$ and each pool $j$, start a separate version space at $\mathcal H$. Its provisional prediction at round $t$ is the center of the version space left by round $t-1$ at $x_t$. After revealing $\boldsymbol\xi_t$, first compute the master's final value $y_t^J$. If $\xi_{t,j}=0$, leave the replay version space unchanged. If $\xi_{t,j}=1$, repeatedly update its current center $a$ toward $y_t^J$ while $|a-y_t^J|>4\eta$: use the plus fiber in \eqref{eq:sign-update-fibers} when $y_t^J>a+4\eta$, and the minus fiber otherwise. Recompute the center after each update and emit the final prediction when the loop stops. We call this procedure the \emph{tolerant replay} of $J$ in pool $j$, and denote its provisional and final predictions by $a_{t,j}^J$ and $y_{t,j}^J$.

For an arbitrary fixed formal code $J\in\mathcal J_*$, the replay need not retain any particular target, since the master's final predictions $y_t^J$ need not approximate any target on that outcome. An update may therefore lead to an empty fiber, in which case the replay enters an absorbing invalid state and outputs zero from that round onward. Every nonempty update strictly decreases rank, so at most $d$ such updates can occur. Hence each correction loop terminates either when the current center is within $4\eta$ of the reference value or when an update produces an empty fiber. In either case, the replay has a well-defined final prediction on every round of every raw outcome.

For a target $h$, apply exactly the same replay rule to its target-driven master values $y_t^h$, and write $a_{t,j}^h$ and $y_{t,j}^h$ for the resulting predictions. These target-driven replays are defined directly from the master for $h$, not by conditioning on any formal code. On an outcome with $J^h\in\mathcal J_*$, they coincide with the replays of $J^h$.

\begin{lemma}[Buffered replay]
\label{lem:buffered-replay}
For every target $h\in\mathcal H$ and pool $j$, the target-driven replay always retains $h$ and makes at most $d$ corrections. More precisely, whenever a correction is performed at center $a$, we have
\begin{equation}
  \label{eq:accuracy-buffered-constraints}
  \begin{cases}
    h(x_t)>a+3\eta,&\text{for a plus update},\\
    h(x_t)<a-3\eta,&\text{for a minus update}.
  \end{cases}
\end{equation}
Moreover, on every round selected by pool $j$,
\begin{equation}
  \label{eq:accuracy-replay-selected}
  |y_{t,j}^h-h(x_t)|\le5\eta.
\end{equation}
\end{lemma}

\begin{proof}
Every pool-selected round is also selected by the master, so $|y_t^h-h(x_t)|\le\eta$ on such rounds.

If the replay performs a plus update at center $a$, then $y_t^h>a+4\eta$, and hence
\begin{equation*}
  h(x_t)\ge y_t^h-\eta>a+3\eta.
\end{equation*}
The minus case is symmetric. Since the corresponding sign-update fibers are defined by the thresholds $a+\eta$ and $a-\eta$, respectively, each correction retains $h$. Every correction strictly decreases rank, so at most $d$ corrections occur.

On a selected round, the replay therefore terminates with $|y_{t,j}^h-y_t^h|\le4\eta$. The triangle inequality gives the claimed error bound.
\end{proof}

On $\mathsf{Cap}$, the target-driven replay in pool $j$ makes at most $d$ corrections, all at selected ordinals in $[m]$. Recording their ordinals and signs gives a formal schedule $J_j^h$ of depth at most $d$. Simulating $J_j^h$ on the same instance sequence and pool $j$'s capped selection history reproduces exactly the replay's version spaces and predictions.

\paragraph{From disagreements to prediction errors.}
For each target and pool, define
\begin{equation}
  \label{eq:accuracy-target-disagreements}
  \begin{aligned}
    D_j^h&:=\sum_{t=1}^T \mathbf 1\{|a_t^h-a_{t,j}^h|>6\eta\},\\
    H_j^h&:=\sum_{t=1}^T\xi_{t,j} \mathbf 1\{|a_t^h-a_{t,j}^h|>6\eta\}.
  \end{aligned}
\end{equation}
Thus $D_j^h$ counts disagreements between the provisional predictions, and $H_j^h$ counts those disagreements selected by pool $j$. On a round selected by pool $j$ with $|a_t^h-a_{t,j}^h|>6\eta$, if the master makes at least one correction, charge the round to the master. Otherwise, $y_t^h=a_t^h$, so $|y_t^h-a_{t,j}^h|>6\eta>4\eta$, and the replay must make at least one correction; charge the round to the replay. Since each procedure makes at most $d$ corrections in total, at most $2d$ distinct rounds can be charged. Hence
\begin{equation}
  \label{eq:accuracy-disagreement-hits}
  H_j^h\le2d \qquad(h\in\mathcal H,\ j\in[k]).
\end{equation}

If $|a_{t,j}^h-h(x_t)|\ge8\eta$ and $|a_t^h-h(x_t)|<2\eta$, then $|a_t^h-a_{t,j}^h|>6\eta$ by the triangle inequality. Consequently, on $\mathsf{Master}$,
\begin{equation}
  \label{eq:accuracy-replay-error-transfer}
  \#\{t:|a_{t,j}^h-h(x_t)|\ge8\eta\} <\frac b{32}+D_j^h.
\end{equation}
By \eqref{eq:accuracy-replay-selected}, a final replay error of size at least $8\eta$ can occur only on an unselected round. There the provisional and final predictions agree, so final outliers are a subset of provisional outliers. In particular, on $\mathsf{Master}$, the condition $D_j^h<b/16$ makes both outlier counts less than $3b/32<b/8$.

\subsection{Simultaneous Accuracy Across Pools}
\label{subsec:accuracy-simultaneous}

Although the replay histories need not be independent, the current disagreement indicators are determined by the past, while the fresh pool selection bits remain independent Bernoulli-$\beta$ variables even after conditioning on the past. We use this fact to establish a joint bound.

\paragraph{A joint bound for fixed master codes.}
For a fixed $J\in\mathcal J_*$ and its replays defined above, set
\begin{equation}
  \label{eq:accuracy-fixed-disagreements}
  \Delta_{t,j}^J:= \mathbf 1\{|a_t^J-a_{t,j}^J|>6\eta\},\qquad D_j^J:=\sum_{t=1}^T\Delta_{t,j}^J,\qquad H_j^J:=\sum_{t=1}^T\xi_{t,j}\Delta_{t,j}^J.
\end{equation}
Every current indicator $\Delta_{t,j}^J$ is determined before $\boldsymbol\xi_t$ is revealed. We do not assert $H_j^J\le2d$ for arbitrary formal codes; that bound is needed only for target-driven histories, as in \eqref{eq:accuracy-disagreement-hits}.

\begin{lemma}[Joint disagreement bound]
\label{lem:joint-disagreements}
For every fixed $J\in\mathcal J_*$ and $I\subseteq[k]$,
\begin{equation}
  \label{eq:accuracy-joint-tail}
  \Pr\left\{D_j^J\ge\frac b{16},\ H_j^J\le2d \text{ for every }j\in I\right\} \le e^{-\alpha |I|d/64}.
\end{equation}
\end{lemma}
\begin{proof}
Work under the original joint law of the raw pool bits. Let $D_{j,t}^J$ and $H_{j,t}^J$ denote the corresponding counts through round $t$, and define
\begin{equation*}
  Z_t:=\exp\!\left( \frac\beta2\sum_{j\in I}D_{j,t}^J -(\log2)\sum_{j\in I}H_{j,t}^J\right), \qquad Z_0=1.
\end{equation*}
Conditional on $\omega_0$ and $\boldsymbol\xi_{<t}$, all current disagreement indicators are fixed and the current pool bits are independent Bernoulli-$\beta$ variables. Hence
\begin{equation*}
  \mathbb E\!\left[ \frac{Z_t}{Z_{t-1}}\,\middle|\, \omega_0,\boldsymbol\xi_{<t}\right] =\prod_{\substack{j\in I:\Delta_{t,j}^J=1}} e^{\beta/2}(1-\beta/2) \le1,
\end{equation*}
where the last inequality uses $1-z\le e^{-z}$. Thus $(Z_t)$ is a nonnegative supermartingale for the filtration generated by $\omega_0,\boldsymbol\xi_1,\ldots,\boldsymbol\xi_t$, and $\mathbb E Z_T\le1$. On the event in \eqref{eq:accuracy-joint-tail},
\begin{equation*}
  Z_T\ge \exp\!\left[|I|d\left(\frac\alpha{32}-2\log2\right)\right],
\end{equation*}
since $\beta b=\alpha d$. Markov's inequality gives the upper bound $\exp[-|I|d(\alpha/32-2\log2)]$. Finally, $\alpha\ge128$ implies $\alpha/32-2\log2\ge\alpha/64$.
\end{proof}

Define $\mathsf{Joint}$ to be the event that there are no $J\in\mathcal J_*$ and $I\subseteq[k]$ with $|I|=k/4$ such that $D_j^J\ge b/16$ and $H_j^J\le2d$ for every $j\in I$. A union bound in Lemma~\ref{lem:joint-disagreements} gives
\begin{equation}
  \label{eq:accuracy-joint-union}
  \Pr(\mathsf{Joint}^{\mathsf c}) \le Q_d(M_*)\binom{k}{k/4}e^{-\alpha kd/256}.
\end{equation}
This estimate does not condition on $\mathsf{Cap}$ or $\mathsf{Master}$, and it assumes no independence between replays.

\begin{proposition}[Simultaneous replay accuracy]
\label{prop:multi-pool-accuracy}
For a sufficiently large universal choice of $\alpha$, in the regime $b>32\alpha d$, the event
\begin{equation*}
  \mathsf{Good}:=\mathsf{Cap}\cap\mathsf{Master} \cap\mathsf{Joint}
\end{equation*}
has probability greater than $15/16$ for every input rule \eqref{eq:accuracy-input-rule}. On this event, for every $h\in\mathcal H$, at least $3k/4$ pools $j$ satisfy
\begin{equation}
  \label{eq:accuracy-many-good-pools}
  \max\!\left\{ \#\{t:|a_{t,j}^h-h(x_t)|\ge8\eta\},\, \#\{t:|y_{t,j}^h-h(x_t)|\ge8\eta\} \right\}<\frac b8.
\end{equation}
\end{proposition}

\begin{proof}
We first prove the accuracy assertion on $\mathsf{Good}$. Fix an outcome in this event and a target $h\in\mathcal H$. Since $\mathsf{Cap}$ holds, the master's correction code $J^h$ belongs to $\mathcal J_*$. On this realized history, the simulations associated with $J^h$ reproduce exactly the target-driven master and its pool replays. Hence
\begin{equation*}
  D_j^{J^h}=D_j^h,\qquad H_j^{J^h}=H_j^h \qquad(j\in[k]).
\end{equation*}

By \eqref{eq:accuracy-disagreement-hits}, $H_j^h\le2d$ for every pool $j$. Since $\mathsf{Joint}$ holds, fewer than $k/4$ pools can satisfy $D_j^h\ge b/16$. Consequently, the set
\begin{equation*}
  I_h:=\{j\in[k]:D_j^h<b/16\}
\end{equation*}
has cardinality at least $3k/4$.

For every $j\in I_h$, the error transfer \eqref{eq:accuracy-replay-error-transfer} and $\mathsf{Master}$ give
\begin{equation*}
  \#\{t:|a_{t,j}^h-h(x_t)|\ge8\eta\} <\frac b{32}+D_j^h <\frac{3b}{32}<\frac b8.
\end{equation*}
Every final replay outlier is also a provisional outlier, as established above. Thus \eqref{eq:accuracy-many-good-pools} holds for every $j\in I_h$. Since the outcome and $h$ were arbitrary, this conclusion holds for all targets on every outcome in $\mathsf{Good}$.

It remains to show that $\Pr(\mathsf{Good})>15/16$. We verify that a sufficiently large universal $\alpha$ makes each of the three failure probabilities small enough.

Recall that $R>32$, $d\ge1$, and $k=4\lceil\log_2R\rceil$. The parameter definitions give
\begin{equation}
  \label{eq:accuracy-entropy-parameters}
  24\le k\le R,\qquad \log R\le\frac{k\log2}{4},\qquad m\le3\alpha dR,\qquad M_*\le12\alpha kdR.
\end{equation}
In the union regime, the bounds $k\beta/2\le u\le k\beta$ further imply
\begin{equation*}
  d\le M=\lceil4uT\rceil\le5\alpha kdR, \qquad ub\ge\frac{\alpha kd}{2}, \qquad uT\ge\frac{\alpha kdR}{2}.
\end{equation*}
Lemma~\ref{lem:count} and \eqref{eq:S-bound} therefore yield
\begin{equation}
  \label{eq:accuracy-code-counts}
  Q_d(M_*)\le(26e\alpha kR)^d, \qquad S_d(M)\le(5e\alpha kR)^d,
\end{equation}
where the second bound is needed only in the union regime.

Set $A_\alpha:=\log(100e\alpha)$. Since $k\le R$ and $\log R\le k\log2/4$, we have
\begin{equation*}
  \log(kR)\le2\log R\le\frac{k\log2}{2}.
\end{equation*}
Together with \eqref{eq:accuracy-code-counts}, $d\ge1$, and $\binom{k}{k/4}\le2^k$, this gives
\begin{align*}
  \log\!\bigl(4Q_d(M_*)S_d(M)2^d\bigr) &\le(2A_\alpha+3)kd,\\
  \log\!\bigl(4Q_d(M_*)\bigr) &\le(A_\alpha+2)kd,\\
  \log\!\left(Q_d(M_*)\binom{k}{k/4}\right) &\le(A_\alpha+2)kd.
\end{align*}

Choose $\alpha\ge128$ sufficiently large that
\begin{equation}
  \label{eq:accuracy-alpha-choice}
  \frac{c_0\alpha}{128}\ge2A_\alpha+4, \qquad \frac{\alpha}{256}\ge A_\alpha+3,
\end{equation}
where $c_0=\log(3/2)/4$. Such a universal choice exists because $A_\alpha=O(\log\alpha)$.

In the union regime, \eqref{eq:accuracy-master-tail} and the preceding estimates imply
\begin{equation*}
  \Pr(\mathsf{Cap}\cap\mathsf{Master}^{\mathsf c}) \le2e^{-kd}.
\end{equation*}
Indeed, the first term in \eqref{eq:accuracy-master-tail} is at most $e^{-kd}$ because
\begin{equation*}
  \frac{c_0ub}{64} \ge\frac{c_0\alpha kd}{128} \ge(2A_\alpha+4)kd,
\end{equation*}
while the second is at most $e^{-kd}$ because
\begin{equation*}
  \frac{2uT}{3} \ge\frac{\alpha kdR}{3} \ge\frac{\alpha kd}{256} \ge(A_\alpha+3)kd.
\end{equation*}
In the all-round regime, $\mathsf{Master}$ holds on every outcome, so the same probability bound holds in both regimes.

Similarly, \eqref{eq:accuracy-joint-union} and \eqref{eq:accuracy-alpha-choice} give
\begin{equation*}
  \Pr(\mathsf{Joint}^{\mathsf c}) \le Q_d(M_*)\binom{k}{k/4}e^{-\alpha kd/256} \le e^{-kd}.
\end{equation*}
Finally, by \eqref{eq:accuracy-cap-tail}, $k\le R$, and $\alpha\ge128$,
\begin{equation*}
  \Pr(\mathsf{Cap}^{\mathsf c}) \le k e^{-\alpha dR/3} \le e^{-R}.
\end{equation*}

Combining these bounds,
\begin{equation*}
  \begin{aligned}
    \Pr(\mathsf{Good}^{\mathsf c}) &\le \Pr(\mathsf{Cap}^{\mathsf c}) +\Pr(\mathsf{Cap}\cap\mathsf{Master}^{\mathsf c}) +\Pr(\mathsf{Joint}^{\mathsf c})\\
    &\le e^{-R}+3e^{-kd} <\frac1{16},
  \end{aligned}
\end{equation*}
where the last inequality uses $R>32$ and $kd\ge24$. The choice of $\alpha$ and all probability bounds are independent of the input rule and the law of its independent auxiliary seed.
\end{proof}

By Lemma~\ref{lem:buffered-replay}, on $\mathsf{Cap}$ each target-driven replay can be reproduced by a sign schedule of depth at most $d$ on the corresponding pool's capped selection history. Thus Proposition~\ref{prop:multi-pool-accuracy} provides at least $3k/4$ accurate schedules for every target on the same event $\mathsf{Good}$.

\section{Robust Identification and Eligible Parents}
\label{sec:identification}

We use the parameters fixed in Section~\ref{sec:accuracy}, in the regime $b>32\alpha d$. A singleton class is covered exactly by its unique member, so assume $|\mathcal H|\ge2$. All simulations in this section use the capped pools. Each selected list therefore has length at most $m$, including on outcomes outside $\mathsf{Cap}$.

The preceding section provides accurate replay histories for every target. We now consider sign schedules whose version spaces reduce to singletons. To obtain useful counting bounds, we additionally require the surviving function to satisfy every correction inequality with a prescribed margin.

\subsection{Robust Identification and Local Counting}
\label{subsec:identification-certificates}

For each integer $0\le r\le d$, let $\mathcal J_r(m)$ denote the formal sign schedules of depth at most $r$ with selected ordinals in $[m]$. Thus $|\mathcal J_r(m)|=Q_r(m)$. In each capped pool $j$, simulate every $L\in\mathcal J_d(m)$ from the initial version space $\mathcal H$, using the rules of Section~\ref{subsec:foundations-schedules}. Write $V_t^{j,L}$ for the version space after round $t$ and $p_t^{j,L}$ for the final prediction on that round. An invalid simulation has $V_t^{j,L}=\varnothing$ and $p_t^{j,L}=0$ from the invalidating update onward. These are simulations of prescribed sign schedules: they do not use master reference values or target values.

\paragraph{The inequalities imposed by a schedule.}
Consider a schedule $L=((s_1,\sigma_1),\ldots,(s_i,\sigma_i))$ that is valid and born by round $t$. Recall that this means that all its prescribed updates have occurred without producing an empty version space. For each update $\ell\in[i]$, let $t_\ell\le t$ be its actual round and let $a_\ell$ be the center immediately before that update. Several $t_\ell$ may coincide; their centers are recomputed after each update. By the definition of the sign-update fibers,
\begin{equation}
  \label{eq:identification-update-constraints}
  V_t^{j,L} =\left\{f\in\mathcal H:
  \begin{array}{ll}
    f(x_{t_\ell})\ge a_\ell+\eta &\text{for every $\ell$ with $\sigma_\ell=+$},\\
    f(x_{t_\ell})\le a_\ell-\eta &\text{for every $\ell$ with $\sigma_\ell=-$}
  \end{array}
  \right\}.
\end{equation}
Here $f$ ranges over candidate functions in $\mathcal H$. We call these inequalities the \emph{update constraints}. Their thresholds are computed from the instance history, the pool selections, and $L$; they are not additional entries in the code.

\begin{definition}[Robust certificate]
\label{def:robust-certificate}
The simulation of $L$ in pool $j$ is a \emph{robust certificate} for $h\in\mathcal H$ by round $t$ if it is valid and born by $t$, its version space is $V_t^{j,L}=\{h\}$, and every update satisfies
\begin{equation}
  \label{eq:identification-certificate-margin}
  \begin{cases}
    h(x_{t_\ell})\ge a_\ell+3\eta,&\text{if $\sigma_\ell=+$},\\
    h(x_{t_\ell})\le a_\ell-3\eta,&\text{if $\sigma_\ell=-$}.
  \end{cases}
\end{equation}
\end{definition}

Thus, for a plus update, $h(x_{t_\ell})$ exceeds the retained boundary $a_\ell+\eta$ by at least $2\eta$; for a minus update it lies at least $2\eta$ below $a_\ell-\eta$. The test first examines the simulation's version space. If this space is a singleton, its unique member specifies the function $h$ whose values are used in \eqref{eq:identification-certificate-margin}. No externally supplied target labels are needed. These conditions are defined set-theoretically; we do not require them to be efficiently computable.

\paragraph{Identification and activation.}
Say that pool $j$ \emph{identifies} $h$ by round $t$ if one of its simulations is a robust certificate for $h$ by that round. Let
\begin{equation}
  \label{eq:identification-pool-sets}
  \mathcal I_j(t):= \{h\in\mathcal H:\text{pool $j$ identifies $h$ by round $t$}\}.
\end{equation}
Identification is tested after the schedule updates on round $t$. Once a schedule is born, its version space and update constraints no longer change. Its certificate therefore persists. In particular, $\mathcal I_j(t)\subseteq\mathcal I_j(t+1)$, and each schedule can certify at most one function over the entire run. Hence $|\mathcal I_j(t)|\le Q_d(m)$. At time zero only the empty schedule is born, and $|\mathcal H|\ge2$, so $\mathcal I_j(0)=\varnothing$.

A function $h$ \emph{activates} on the first round when at least $k/2$ pools have identified it. Formally, define
\begin{equation}
  \label{eq:identification-activation-time}
  \tau(h):=\min\left\{t\in[T]: \#\{j\in[k]:h\in\mathcal I_j(t)\}\ge k/2\right\},
\end{equation}
with $\tau(h)=\infty$ if the set is empty. Write $\mathcal A_t:=\{h\in\mathcal H:\tau(h)\le t\}$ and $\mathcal A:=\mathcal A_T$. These sets are finite, since $\mathcal A_t\subseteq\bigcup_{j=1}^k\mathcal I_j(t)$.

\paragraph{Singleton replay states.}
Write $U_{t,j}^h$ for the version space of the target-driven replay for $h$ in raw pool $j$ after round $t$. On $\mathsf{Cap}$, record only its corrections through round $t$, obtaining a prefix schedule in $\mathcal J_d(m)$. In capped pool $j$, this prefix schedule is valid and born by $t$ and reproduces the replay through that round. Lemma~\ref{lem:buffered-replay} gives the strict versions of \eqref{eq:identification-certificate-margin} for all its updates. Consequently,
\begin{equation}
  \label{eq:identification-singleton-replay}
  \text{on $\mathsf{Cap}$},\qquad U_{t,j}^h=\{h\}\quad\Longrightarrow\quad h\in\mathcal I_j(t).
\end{equation}
This implication uses the buffered update inequalities, not the accuracy of the replay on unselected rounds.

\paragraph{Local identification bounds.}

The separation in \eqref{eq:identification-certificate-margin} limits how many functions a pool can identify within a subclass of small $\operatorname{sfat}_\eta(V)$. The bound holds for every realized capped history and does not require $\mathsf{Good}$. 

\begin{lemma}[Local identification bound]
\label{lem:local-identification}
On every outcome, for every $t\in\{0,\ldots,T\}$, pool $j\in[k]$, and nonempty subclass $V\subseteq\mathcal H$ with $\operatorname{sfat}_\eta(V)\le r$, where $r\in\{0,\ldots,d\}$,
\begin{equation}
  \label{eq:identification-local-pool}
  |\mathcal I_j(t)\cap V|\le Q_r(m).
\end{equation}
\end{lemma}
\begin{proof}
Fix the outcome, $t$, $j$, and $V$. We first show that any two distinct functions $h,h'\in\mathcal I_j(t)$ differ by more than $2\eta$ at some round selected by capped pool $j$ by time $t$. Choose a certificate for $h$. Since it is born by round $t$ and its version space is $\{h\}$, the function $h'$ must violate an update constraint imposed at some round $s\le t$ selected by capped pool $j$.

If this is a plus update at center $a$, then
\begin{equation*}
  h'(x_s)<a+\eta,\qquad h(x_s)\ge a+3\eta,
\end{equation*}
so $h(x_s)-h'(x_s)>2\eta$. The minus case is symmetric.

Let $z_1,\ldots,z_n$ be the instances in that list through round $t$, in order, so $n\le m$. If $n=0$, no pool schedule other than the empty schedule is born, and $\mathcal I_j(t)=\varnothing$. Otherwise, apply Lemma~\ref{lem:pointwise} to the initial class $V$ and the sequence $z_1,\ldots,z_n$, selecting every position. It gives at most $Q_r(m)$ prediction vectors such that every $f\in V$ is within $\eta$ of one vector at every position. Two distinct functions in $\mathcal I_j(t)\cap V$ cannot use the same vector: the triangle inequality would make their values differ by at most $2\eta$ at every selected position, contrary to the separation just proved. This proves \eqref{eq:identification-local-pool}.
\end{proof}

\begin{corollary}[Local activation bound]
\label{cor:local-activation}
Under the same hypotheses on $V$ and $r$, on every outcome and for every $t\in\{0,\ldots,T\}$,
\begin{equation}
  \label{eq:identification-local-activation}
  |\mathcal A_t\cap V|\le2Q_r(m).
\end{equation}
In particular, $|\mathcal A|\le2Q_d(m)$.
\end{corollary}
\begin{proof}
Every $h\in\mathcal A_t$ has been identified by at least $k/2$ pools by time $t$, since identification persists after activation. Counting the pairs $(h,j)$ with $h\in\mathcal A_t\cap V$ and $h\in\mathcal I_j(t)$ gives
\begin{equation*}
  \frac{k}{2}|\mathcal A_t\cap V| \le\sum_{j=1}^k|\mathcal I_j(t)\cap V| \le kQ_r(m).
\end{equation*}
The first assertion follows by division, and the second by taking $t=T$, $V=\mathcal H$, and $r=d$.
\end{proof}

In particular, $V$ may be a version space produced during
the run; it need not be fixed before the experiment.

\subsection{Parent and Child Experts}
\label{subsec:identification-experts}

We preallocate two types of prediction streams:
\emph{parents}, which follow fixed sign schedules, and
\emph{children}, which may later be assigned to activated
functions. For every pool and every schedule of depth at
most $d-1$, retain its prediction stream as a parent.
Thus the set of parent indices is
\begin{equation}
  \label{eq:identification-parent-indices}
  \mathcal P:=[k]\times\mathcal J_{d-1}(m), \qquad P:=|\mathcal P|=kQ_{d-1}(m).
\end{equation}
For each $v=(j,L)\in\mathcal P$, let $p_t^v$ denote the prediction of schedule $L$ in capped pool $j$ at round $t$, and call the depth of $L$ the depth of $v$. The parent always follows this schedule. It outputs zero if the schedule becomes invalid.

If a depth-$i$ parent becomes valid and born, let $V_v$ be its version space after birth. This space is fixed for the rest of the run, and the rank-drop property gives
\begin{equation}
  \label{eq:identification-parent-rank}
  \operatorname{sfat}_\eta(V_v)\le d-i.
\end{equation}
Identification considers all schedules in $\mathcal J_d(m)$, whereas parent indices are allocated only to schedules in $\mathcal J_{d-1}(m)$.

\paragraph{Children and their outputs.}
For each $i\in\{0,\ldots,d-1\}$, fix a nonnegative integer $c_i$ before the experiment. Attach $c_i$ additional prediction indices to every depth-$i$ parent $v$, and call them its \emph{children}, indexed by $(v,\ell)$ for $\ell\in\{1,\ldots,c_i\}$. Each child initially copies its parent's outputs. On a round when a function $h$ activates, at most one unused child may be assigned to $h$. From that round onward it predicts $h(x_s)$ on every subsequent instance $x_s$, and it is never reassigned.

Assignments on round $t$ are made after the capped schedule updates and identification tests, but before emitting the child predictions on that round. Thus
\begin{equation}
  \label{eq:identification-child-output}
  g_t^{v,\ell}:=
  \begin{cases}
    h(x_t),&\text{if this child has been assigned to $h$ by round $t$},\\
    p_t^v,&\text{if this child is still unassigned}.
  \end{cases}
\end{equation}
In particular, assignment does not change any past prediction.

The capacities $c_i$ and the rule choosing the assignments will be specified in the next section. Here we define which choices are permitted and prove their deterministic properties. For any such capacities, the number of parent and child indices is
\begin{equation}
  \label{eq:identification-total-indices}
  K=P+\sum_{i=0}^{d-1}k n_i(m)c_i.
\end{equation}
This counts invalid schedules and unused children as well. All these indices are present from the start. The other schedule simulations are internal calculations, not additional members of the covering family.

\subsection{Eligibility and Many Compatible Parents}
\label{subsec:identification-eligibility}

An unused child can cover a newly activated function $h$ if its parent has already made few errors against $h$: after assignment, the child predicts $h$ exactly. We therefore test the parent's actual predictions before the activation round. For $v\in\mathcal P$, $h\in\mathcal H$, and $t\in[T]$, write
\begin{equation}
  \label{eq:identification-past-errors}
  E_{t-1}(v,h):= \sum_{s=1}^{t-1}\mathbf 1\{|p_s^v-h(x_s)|\ge8\eta\}.
\end{equation}
The sum is zero at $t=1$.

\begin{definition}[Eligible parent]
\label{def:eligible-parent}
Suppose $h$ activates on round $t$. A parent $v$ is \emph{eligible} for this activation if its schedule is valid and born by round $t-1$, and
\begin{equation}
  \label{eq:identification-eligibility-tests}
  h\in V_v,\qquad |V_v|\ge2,\qquad E_{t-1}(v,h)<b/8.
\end{equation}
Let $\mathcal E_j(h,t)$ be the set of eligible parents in pool $j$. Fix an arbitrary ordering of the sign schedules in $\mathcal J_{d-1}(m)$. For each pool $j$ with $\mathcal E_j(h,t)\ne\varnothing$, choose the eligible parent whose schedule comes first in this ordering. Let $\mathcal E(h,t)$ be the set of these chosen parents.
\end{definition}

We permit an activation of $h$ on round $t$ to use only a parent from $\mathcal E(h,t)$, and only an unused child of that parent.

The sets in this definition can be determined on round $t$. The function $h$ is known from its certificates at activation; its values on earlier instances can therefore be evaluated. The parent states and all predictions used in \eqref{eq:identification-past-errors} are already available. Neither the test nor the choice of a representative from each pool uses future instances or future accuracy. In particular, $\mathsf{Good}$ is not part of the eligibility test. Eligibility also does not depend on whether a parent's children have already been used; a parent with no unused child still belongs to the same eligible sets.

If an unused child of $v\in\mathcal E(h,t)$ is assigned to $h$ on round $t$, then \eqref{eq:identification-child-output} gives
\begin{equation}
  \label{eq:identification-child-accuracy}
  \sum_{s=1}^T\mathbf 1\{|g_s^{v,\ell}-h(x_s)|\ge8\eta\} =E_{t-1}(v,h)<b/8.
\end{equation}
Indeed, before $t$ the child copied its parent, and from $t$ onward it predicts $h$ exactly. Requiring assignments to use only the information just described makes these prediction rules causal in the instances and the revealed random choices.

There is also a uniform bound on how many distinct activations can have a given parent among their choices. Fix a valid, born parent $v$ of depth $i$. Every function $h$ that activates and has $v\in\mathcal E(h,\tau(h))$ belongs to $\mathcal A\cap V_v$, since eligibility requires $h\in V_v$. By the local activation bound, the number of such functions is at most $2Q_{d-i}(m)$. Hence \eqref{eq:identification-parent-rank} and Corollary~\ref{cor:local-activation} give
\begin{equation}
  \label{eq:identification-parent-requests}
  \#\{h\in\mathcal A:v\in\mathcal E(h,\tau(h))\} \le |\mathcal A\cap V_v| \le2Q_{d-i}(m).
\end{equation}
A parent that never becomes valid and born is never eligible. This bound holds on every outcome, not only on $\mathsf{Good}$, and allows the later input instances to depend on earlier assignment decisions.

\paragraph{Many eligible parents at activation.}

We now use the simultaneous replay accuracy from Proposition~\ref{prop:multi-pool-accuracy} to show that the eligible sets are large on $\mathsf{Good}$. 

\begin{lemma}[Many eligible parents]
\label{lem:many-eligible-parents}
On $\mathsf{Good}$, every function $h$ that activates on a round $t\in[T]$ satisfies
\begin{equation}
  \label{eq:identification-many-parents}
  |\mathcal E(h,t)|\ge k/4.
\end{equation}
\end{lemma}
\begin{proof}
Fix an outcome in $\mathsf{Good}$ and a function $h$ activating on round $t$. Let $G_h\subseteq[k]$ be the pools whose replays satisfy \eqref{eq:accuracy-many-good-pools}. By Proposition~\ref{prop:multi-pool-accuracy}, $|G_h|\ge3k/4$. Before round $t$, fewer than $k/2$ pools had identified $h$; otherwise it would already have activated. Therefore
\begin{equation*}
  \#\{j\in G_h:h\notin\mathcal I_j(t-1)\} \ge |G_h|-\#\{j:h\in\mathcal I_j(t-1)\} >k/4.
\end{equation*}

Fix a pool $j$ in this set, and let $L$ consist of all corrections made by its target-driven replay through round $t-1$. On $\mathsf{Cap}$ these have ordinals in $[m]$, and the capped-pool simulation of $L$ is valid, born by $t-1$, and has version space $U_{t-1,j}^h$. This space contains $h$. It cannot equal $\{h\}$, since \eqref{eq:identification-singleton-replay} would imply $h\in\mathcal I_j(t-1)$.

Since $\mathcal H$ is uniformly greater-than-$\eta$ separated, every nonsingleton subclass of $\mathcal H$ has rank at least one. If $L$ has depth $i$, the rank-drop property consequently gives
\begin{equation*}
  1\le\operatorname{sfat}_\eta(U_{t-1,j}^h)\le d-i, \qquad\text{so }i\le d-1.
\end{equation*}
Thus $v=(j,L)$ is one of the preallocated parents. Its fixed version space is $V_v=U_{t-1,j}^h$, which contains $h$ and is not a singleton. By the prefix property in Section~\ref{subsec:foundations-schedules},
\begin{equation*}
  p_s^v=y_{s,j}^h\qquad(1\le s<t).
\end{equation*}
Since $j\in G_h$, it follows that
\begin{equation*}
  E_{t-1}(v,h) =\sum_{s<t}\mathbf 1\{|y_{s,j}^h-h(x_s)|\ge8\eta\} \le\sum_{s=1}^T\mathbf 1\{|y_{s,j}^h-h(x_s)|\ge8\eta\} <b/8.
\end{equation*}

All eligibility conditions hold, so $\mathcal E_j(h,t)$ is nonempty. This applies to more than $k/4$ distinct pools. Since $\mathcal E(h,t)$ contains one eligible parent from each such pool, we obtain \eqref{eq:identification-many-parents}.
\end{proof}

The parent used in this proof is the separate schedule containing only the replay's corrections before round $t$. Its version space remains fixed, even if the full replay makes further corrections or identifies $h$ on round $t$. Only their predictions before $t$ must agree: a child assigned on round $t$ switches to $h$ before emitting its prediction on that round.

\begin{corollary}[Targets that do not activate]
\label{cor:unactivated-parent}
On $\mathsf{Good}$, every $h\in\mathcal H\setminus\mathcal A$ has a parent $v\in\mathcal P$ with
\begin{equation}
  \label{eq:identification-unactivated-cover}
  \sum_{t=1}^T\mathbf 1\{|p_t^v-h(x_t)|\ge8\eta\}<b/8.
\end{equation}
\end{corollary}
\begin{proof}
Fewer than $k/2$ pools identify $h$ by time $T$, while at least $3k/4$ replays satisfy \eqref{eq:accuracy-many-good-pools}. Choose a pool $j$ with an accurate replay that has not identified $h$. Its final replay version space is not a singleton, by \eqref{eq:identification-singleton-replay}. The same rank argument as above shows that its entire replay code $J_j^h$ has depth at most $d-1$. On $\mathsf{Cap}$ this code has ordinals in $[m]$, so $v=(j,J_j^h)$ is a parent and reproduces the replay on all $T$ rounds. Its final-prediction accuracy gives \eqref{eq:identification-unactivated-cover}.
\end{proof}

Thus a target that never activates is covered by a permanent parent. A target that activates has at least $k/4$ eligible parents on $\mathsf{Good}$, and assigning an unused child of any one of them gives \eqref{eq:identification-child-accuracy}. It remains to choose the capacities and distribute the activations among these parents so that, with positive probability jointly with $\mathsf{Good}$, every activation receives an unused child and the count in \eqref{eq:identification-total-indices} is at most $(CR)^d$.

\section{Shared-Parent Routing and Load Balancing}
\label{sec:routing}

We now specify how an activating function receives a child of one of its eligible parents. The rule chooses one parent uniformly from the retained set $\mathcal E(h,t)$ when this set has at least $k/4$ members. We bound how often a parent is chosen and use this bound to set the child capacities $c_i$. The detailed design of the routing mechanism and the tuning of its
parameters were developed through discussions and iterative refinement
involving the authors and ChatGPT 6 Pro; see the AI disclosure at the
end of the paper.

\subsection{The Causal Routing Protocol}
\label{subsec:routing-protocol}

For now, let the capacities $c_0,\ldots,c_{d-1}$ be any fixed nonnegative integers, as in Section~\ref{subsec:identification-experts}. By Corollary~\ref{cor:local-activation}, the total number of functions that activate during a run is at most
\begin{equation}
  \label{eq:routing-activation-budget}
  A_*:=2Q_d(m).
\end{equation}
We allow $A_*$ routing slots on each round, where one slot processes one newly activated function; slots with no function do nothing. Fix an arbitrary total order on $\mathcal H$. In each round, process the newly activated functions in this order.

For $t\in[T]$ and $a\in[A_*]$, let $\pi_{t,a}$ be a uniformly random permutation of the finite parent set $\mathcal P$. All these permutations are mutually independent and independent of the raw pool bits. A permutation specifies an ordering of all parents. For an activation with $|\mathcal E(h,t)|\ge k/4$, conditional on the preceding history, each member of $\mathcal E(h,t)$ is selected with probability $1/|\mathcal E(h,t)|$. The permutation $\pi_{t,a}$ is revealed only in slot $a$ of round $t$, including when that slot has no function to process.

Write $\boldsymbol\pi_t=(\pi_{t,1},\ldots,\pi_{t,A_*})$ and $\boldsymbol\pi_{<t}=(\boldsymbol\pi_1,\ldots, \boldsymbol\pi_{t-1})$. We allow inputs of the form
\begin{equation}
  \label{eq:routing-input-rule}
  x_t=\widetilde F_t(\boldsymbol\xi_{<t}, \boldsymbol\pi_{<t},\omega_{\mathrm{in}}),
\end{equation}
where the maps $\widetilde F_t$ are deterministic and $\omega_{\mathrm{in}}$ is a finite-support input seed independent of all pool bits and routing permutations. Thus $x_t$ may depend on all earlier selections, assignments, and predictions, but is chosen before the current pool vector and current routing permutations are revealed.

Initialize a counter $L_v=0$ for each parent $v\in\mathcal P$. A parent remains eligible after all its children have been assigned. If routing selects such a parent, its load counter increases and an overload is declared, but no child is assigned. At round $t$, reveal $\boldsymbol\xi_t$, perform the capped schedule updates, and determine the newly activated functions $\mathcal A_t\setminus\mathcal A_{t-1}$. List them in the fixed order as $h_{t,1},\ldots,h_{t,n_t}$, where $n_t\le A_*$. The eligible set $\mathcal E(h_{t,a},t)$ does not depend on the current routing permutation $\pi_{t,a}$.

In each slot $a\in[A_*]$, reveal $\pi_{t,a}$ and proceed as follows. If $a>n_t$, make no change. Otherwise put $h=h_{t,a}$. If $|\mathcal E(h,t)|<k/4$, leave $h$ unassigned. If $|\mathcal E(h,t)|\ge k/4$, choose the first parent $v$ in $\mathcal E(h,t)$ according to $\pi_{t,a}$ and increase $L_v$ by one. Write $\ell$ for its updated value and $i$ for the depth of $v$. When $\ell\le c_i$, assign child $(v,\ell)$ to $h$. When $\ell>c_i$, make no child assignment; this is an \emph{overload} at $v$.

The first $c_i$ choices of a depth-$i$ parent use its children in order, so every assigned child is unused at assignment. Each function is processed only on its activation round and only once. After all slots of round $t$, emit the parent predictions and the child predictions specified in \eqref{eq:identification-child-output}.

We continue the selections and counters after an overload, and also after an activation is left unassigned. Thus $L_v$ counts routing choices, not just successful child assignments. On $\mathsf{Good}$, the case $|\mathcal E(h,t)|<k/4$ cannot occur, by Lemma~\ref{lem:many-eligible-parents}.

The sampling guarantee from the preceding sections applies to \eqref{eq:routing-input-rule}. Indeed, the auxiliary seed
\begin{equation*}
  \omega_0:=\bigl(\omega_{\mathrm{in}}, (\pi_{t,a})_{t\in[T],\,a\in[A_*]}\bigr)
\end{equation*}
has finite support and is independent of all raw pool bits. Given this seed, every earlier assignment is determined by the past raw vectors. Hence the input rule has the form \eqref{eq:accuracy-input-rule}, and Proposition~\ref{prop:multi-pool-accuracy} gives $\Pr(\mathsf{Good})>15/16$ for the routing protocol just defined.

\subsection{Routing Probabilities and Load Bounds}
\label{subsec:routing-requests}

Fix a parent $v$. Consider a routing step $(t,a)$ that processes a newly activated function $h_{t,a}$. Define
\begin{equation}
  \label{eq:routing-request-indicators}
  B_{t,a}^v:= \mathbf 1\bigl\{|\mathcal E(h_{t,a},t)|\ge k/4 \text{ and }v\in\mathcal E(h_{t,a},t)\bigr\},
\end{equation}
and let $X_{t,a}^v$ be the indicator that the routing rule chooses $v$ in this slot, whether or not a child is assigned. For an empty slot, set both indicators to zero. Thus
\begin{equation*}
  N_v:=\sum_{t=1}^T\sum_{a=1}^{A_*}B_{t,a}^v \qquad\text{and}\qquad L_v=\sum_{t=1}^T\sum_{a=1}^{A_*}X_{t,a}^v
\end{equation*}
count, respectively, the activations at which $v$ could be chosen and the activations at which it is chosen.

If $v$ has depth $i$, then on every outcome
\begin{equation}
  \label{eq:routing-request-bound}
  N_v \le \#\{h\in\mathcal A:v\in\mathcal E(h,\tau(h))\} \le 2Q_{d-i}(m).
\end{equation}
The first inequality holds because each function activates only once; the size test in \eqref{eq:routing-request-indicators} can only remove activations from the count. The second inequality is \eqref{eq:identification-parent-requests}. A parent that never becomes valid and born has $N_v=L_v=0$.

To describe the information available at a routing choice, let
\begin{equation*}
  \begin{aligned}
    \mathcal F_{t,0} &:=\sigma\bigl(\omega_{\mathrm{in}}, \boldsymbol\xi_1,\ldots,\boldsymbol\xi_t, \boldsymbol\pi_1,\ldots,\boldsymbol\pi_{t-1}\bigr),\\
    \mathcal F_{t,a} &:=\sigma\bigl(\mathcal F_{t,0}, \pi_{t,1},\ldots,\pi_{t,a}\bigr) \qquad(1\le a\le A_*).
  \end{aligned}
\end{equation*}
Also put $\mathcal F_{0,A_*}=\sigma(\omega_{\mathrm{in}})$. These sigma-fields reveal the raw vector first on each round and then the routing permutations one at a time. The functions $h_{t,a}$ and their eligible sets are already determined by $\mathcal F_{t,0}$. In particular, $B_{t,a}^v$ is known given $\mathcal F_{t,a-1}$, while $\pi_{t,a}$ remains an independent uniform permutation. Consequently,
\begin{equation}
  \label{eq:routing-selection-probability}
  \Pr\{X_{t,a}^v=1\mid\mathcal F_{t,a-1}\} =
  \begin{cases}
    |\mathcal E(h_{t,a},t)|^{-1},& B_{t,a}^v=1,\\
    0,& B_{t,a}^v=0
  \end{cases}
  \le \frac4k B_{t,a}^v.
\end{equation}
For the first case, each member of the eligible set is equally likely to appear first in the fresh permutation. This remains true when $v$ has no unused children, since routing still chooses from the same set and increments $L_v$.

\paragraph{Load supermartingale and child capacities.}

We combine the bound on how often $v$ can be considered with the conditional probability that it is chosen. The routing indicators $X_{t,a}^v$ need not be independent across rounds, since later inputs and eligible sets may depend on earlier routing choices.

\begin{lemma}[Exponential moment of a parent load]
\label{lem:routing-load-moment}
For every fixed parent $v$ of depth $i$, every fixed choice of capacities, and every input rule \eqref{eq:routing-input-rule},
\begin{equation}
  \label{eq:routing-load-moment}
  \mathbb E e^{L_v} \le \exp\!\left(\frac{8(e-1)Q_{d-i}(m)}{k}\right).
\end{equation}
\end{lemma}
\begin{proof}
Let $N_v(t,a)$ and $L_v(t,a)$ be the two counts through slot $a$ of round $t$, including all earlier rounds. Both counts are initially zero. Here $a=0$ denotes the start of round $t$, before any routing steps, and $a=A_*$ denotes the end of that round, after all routing steps. Thus, for $t\ge2$,
\begin{equation*}
  N_v(t,0)=N_v(t-1,A_*),\qquad L_v(t,0)=L_v(t-1,A_*).
\end{equation*}
Initially, $N_v(1,0)=L_v(1,0)=0$. Set $\lambda:=4(e-1)/k$ and define
\begin{equation}
  \label{eq:routing-load-supermartingale}
  Z_{t,a}:=\exp\bigl(L_v(t,a)-\lambda N_v(t,a)\bigr).
\end{equation}
Its initial value is one. At a raw-vector reveal it is unchanged. At a routing slot, the increments of $L_v$ and $N_v$ are $X_{t,a}^v$ and $B_{t,a}^v$, respectively, so
\begin{equation*}
  \frac{Z_{t,a}}{Z_{t,a-1}} =\exp\bigl(X_{t,a}^v-\lambda B_{t,a}^v\bigr).
\end{equation*}
Since $X_{t,a}^v\in\{0,1\}$ and $B_{t,a}^v$ is $\mathcal F_{t,a-1}$-measurable, \eqref{eq:routing-selection-probability} gives
\begin{equation*}
  \begin{aligned}
    \mathbb E\bigl[e^{X_{t,a}^v}\mid\mathcal F_{t,a-1}\bigr] &=1+(e-1)\Pr\{X_{t,a}^v=1\mid\mathcal F_{t,a-1}\}\\
    &\le 1+\lambda B_{t,a}^v \le e^{\lambda B_{t,a}^v}.
  \end{aligned}
\end{equation*}
It follows that
\begin{equation*}
  \mathbb E\!\left[ \frac{Z_{t,a}}{Z_{t,a-1}} \,\middle|\,\mathcal F_{t,a-1}\right] \le1.
\end{equation*}
Thus $Z$ is a nonnegative supermartingale along the finite sequence of raw-vector reveals and routing slots, with respect to the sigma-fields just defined. In particular, $\mathbb E Z_{T,A_*}\le1$.

By \eqref{eq:routing-request-bound}, $N_v\le2Q_{d-i}(m)$ on every outcome. Therefore
\begin{equation*}
  \begin{aligned}
    \mathbb E e^{L_v} &=\mathbb E\bigl[Z_{T,A_*}e^{\lambda N_v}\bigr]\\
    &\le e^{2\lambda Q_{d-i}(m)}\mathbb E Z_{T,A_*}\\
    &\le \exp\!\left(\frac{8(e-1)Q_{d-i}(m)}{k}\right).
  \end{aligned}
\end{equation*}
\end{proof}

We can now choose the number of children for each parent. The exponential moment bound turns an additive $\log(4P)$ in the capacity into a failure probability of at most $1/(4P)$ for each parent. Since $e-1<2$, set
\begin{equation}
  \label{eq:routing-capacities}
  c_i:=\left\lceil \frac{16Q_{d-i}(m)}{k}+\log(4P) \right\rceil, \qquad 0\le i\le d-1.
\end{equation}
These values depend only on the parameters fixed before the experiment.

\begin{proposition}[No overload]
\label{prop:routing-no-overload}
With capacities \eqref{eq:routing-capacities}, define
\begin{equation}
  \label{eq:routing-no-overload}
  \mathsf{NoOverload}:= \{L_v\le c_i\text{ for every depth-$i$ parent }v\in\mathcal P, \ 0\le i\le d-1\}.
\end{equation}
For every input rule \eqref{eq:routing-input-rule},
\begin{equation}
  \label{eq:routing-overload-bound}
  \Pr(\mathsf{NoOverload}^{\mathsf c})\le\frac14.
\end{equation}
\end{proposition}
\begin{proof}
Fix a parent $v$ of depth $i$. Markov's inequality and Lemma~\ref{lem:routing-load-moment} imply
\begin{equation*}
  \begin{aligned}
    \Pr\{L_v>c_i\} &\le e^{-c_i}\mathbb E e^{L_v}\\
    &\le \exp\!\left(-c_i+ \frac{8(e-1)Q_{d-i}(m)}{k}\right)\\
    &\le \exp\!\left(-c_i+ \frac{16Q_{d-i}(m)}{k}\right) \le\frac1{4P}.
  \end{aligned}
\end{equation*}
Taking a union bound over the $P$ parent indices proves \eqref{eq:routing-overload-bound}.
\end{proof}

The overload bound is under the full joint law of sampling and routing, without conditioning on $\mathsf{Cap}$ or $\mathsf{Good}$. On $\mathsf{Good}\cap\mathsf{NoOverload}$, every activating function receives an unused child on its activation round: its eligible set has at least $k/4$ members, and the selected parent's counter never exceeds its capacity.

\section{Expert Counting and Global Determinization}
\label{sec:determinization}

For the counting and probability arguments, retain the separated class $\mathcal H$ and the regime $d\ge1$, $|\mathcal H|\ge2$, and $b>32\alpha d$ from the preceding section. The child capacities are now fixed by \eqref{eq:routing-capacities}.

\subsection{Expert Counting and Simultaneous Coverage}
\label{subsec:determinization-counting}

\begin{lemma}[Expert count]
\label{lem:expert-count}
There is a universal constant $C_{\mathrm{cnt}}\ge1$ such that the number $K$ of preallocated parent and child indices satisfies
\begin{equation}
  \label{eq:determinization-count-bound}
  K\le(C_{\mathrm{cnt}}R)^d.
\end{equation}
\end{lemma}
\begin{proof}
There are $kn_i(m)$ parents of depth $i$, each with $c_i$ children. Using \eqref{eq:identification-total-indices}, \eqref{eq:routing-capacities}, and $\lceil z\rceil\le z+1$, we obtain
\begin{align}
  \label{eq:determinization-count-decomposition}
  K &=P+\sum_{i=0}^{d-1}kn_i(m)c_i\notag\\
  &\le16\sum_{i=0}^{d-1}n_i(m)Q_{d-i}(m) +P\bigl(\log(4P)+2\bigr)\notag\\
  &\le16Q_d(2m)+P\bigl(\log(4P)+2\bigr).
\end{align}
The factor $k$ from the $kn_i(m)$ depth-$i$ parents cancels with the $1/k$ in the leading term of $c_i$. The term $P(\log(4P)+2)$ accounts for the permanent parents, the logarithmic capacity overhead, and the rounding overhead. The last inequality is \eqref{eq:triangular-convolution}.

Put $C_1:=14e\alpha$. Since $m\le3\alpha dR$, \eqref{eq:Q-bound} gives
\begin{equation*}
  Q_d(2m)\le\bigl[2e(1+6\alpha R)\bigr]^d \le(C_1R)^d.
\end{equation*}
For $d\ge2$, the inequality $d/(d-1)\le2$ similarly gives
\begin{equation*}
  Q_{d-1}(m) \le\left[2e\left(1+\frac{3\alpha dR}{d-1}\right)\right]^{d-1} \le(C_1R)^{d-1}.
\end{equation*}
For $d=1$, the same final bound holds because $Q_0(m)=1$. Thus, in all cases,
\begin{equation}
  \label{eq:determinization-parent-count}
  P=kQ_{d-1}(m)\le k(C_1R)^{d-1}.
\end{equation}

We next bound the second term in \eqref{eq:determinization-count-decomposition}. The definition $k=4\lceil\log_2R\rceil$ and $R>32$ imply $k\le8\log R$, and \eqref{eq:accuracy-entropy-parameters} gives $k\le R$. Consequently,
\begin{equation*}
  \log(4P)+2 \le\log(4k)+2+(d-1)\log(C_1R) \le C_2d\log R,
\end{equation*}
where we may take the universal constant $C_2:=1+(\log(4C_1)+2)/\log32$, using $\log k\le\log R$ and $\log R>\log32$. It follows that
\begin{align}
  \label{eq:determinization-overhead}
  P\bigl(\log(4P)+2\bigr) &\le8C_2d(\log R)^2(C_1R)^{d-1}\notag\\
  &\le8C_2dR(C_1R)^{d-1}\notag\\
  &\le(C_3R)^d
\end{align}
for a universal $C_3$. The second inequality uses $(\log R)^2\le R$ for $R>32$; the last follows from $d\le2^d$, for example with $C_3=2\max\{C_1,8C_2\}$. Combining these bounds in \eqref{eq:determinization-count-decomposition}, we may take $C_{\mathrm{cnt}}=16C_1+C_3$. All these constants are universal because $\alpha$ was fixed universally in \eqref{eq:accuracy-alpha-choice}.
\end{proof}

\paragraph{Simultaneous coverage.}

Number the $K$ preallocated parent and child indices by $\nu\in[K]$, and let $w_t^\nu$ denote the prediction of index $\nu$ at round $t$ under the randomized sampling and routing protocol.

\begin{lemma}[Simultaneous coverage]
\label{lem:simultaneous-coverage}
For every input rule \eqref{eq:routing-input-rule}, with probability greater than $1/2$, the outputs satisfy
\begin{equation}
  \label{eq:determinization-realized-cover}
  \forall h\in\mathcal H\;\exists\nu\in[K]:\qquad \sum_{t=1}^T\mathbf 1\{|w_t^\nu-h(x_t)|\ge8\eta\}<b/8.
\end{equation}
\end{lemma}
\begin{proof}
The application of Proposition~\ref{prop:multi-pool-accuracy} in Section~\ref{subsec:routing-protocol} and Proposition~\ref{prop:routing-no-overload} give
\begin{align}
  \label{eq:determinization-success-probability}
  \Pr(\mathsf{Good}\cap\mathsf{NoOverload}) &\ge1-\Pr(\mathsf{Good}^{\mathsf c}) -\Pr(\mathsf{NoOverload}^{\mathsf c})\notag\\
  &>1-\frac1{16}-\frac14=\frac{11}{16}>\frac12.
\end{align}
Fix an outcome in this intersection and $h\in\mathcal H$. If $h$ activates, the routing protocol assigns it an unused child of a parent in $\mathcal E(h,\tau(h))$, as proved at the end of Section~\ref{sec:routing}. By \eqref{eq:identification-child-accuracy}, this child's index supplies the required witness for $h$. If $h$ does not activate, Corollary~\ref{cor:unactivated-parent} supplies a parent with the same error bound. In either case, the witnessing index belongs to the preallocated family. Since $h$ was arbitrary, the one event in \eqref{eq:determinization-success-probability} gives the stated coverage for every target.
\end{proof}

\subsection{Game Determinization and Main Upper Bound}
\label{subsec:determinization-game}

Lemma~\ref{lem:simultaneous-coverage} gives successful random choices for each fixed input rule. Choosing one successful realization separately for each instance sequence, however, need not give the same predictions on shared prefixes. To obtain a global family, we need one deterministic rule for choosing the pool selections and routing permutations after each observed input prefix. The guarantee against adaptive input rules allows us to obtain such a rule by the following game argument.

\begin{lemma}[Finite-horizon determinization]
\label{lem:finite-horizon-determinization}
Consider a game of $T$ rounds between an adversary and a constructor. In each round the adversary first chooses an input from a nonempty set, and the constructor then chooses an action from a nonempty set. The legal choices may depend on the preceding moves. Both players observe all earlier moves, and the constructor also observes the current input. Each terminal history is designated winning or losing for the constructor.

If one randomized constructor strategy wins with positive probability against every deterministic adversary strategy, then there is a deterministic constructor strategy that wins against every adversary strategy.
\end{lemma}
\begin{proof}
Label each terminal history according to the prescribed winning condition. Working backwards through nonterminal histories, label a constructor decision node as winning if some legal action leads to a winning successor. Label a history at which the adversary is to move as winning if every legal input leads to a winning history. All other histories are losing. The recursion is well-defined because the game has finite depth, even if its branching is infinite.

Suppose that the initial history were losing. At each losing history where the adversary moves, there is an input leading to a losing history. Choose one such input at each of these histories, and choose arbitrary legal inputs at the adversary's other histories. At a losing history where the constructor moves, every legal action leads to a losing history. The chosen adversary strategy therefore forces a losing terminal history for every realization of the constructor's random choices. Its winning probability against this deterministic adversary would be zero, a contradiction.

The initial history is consequently winning. At each winning history where the constructor moves, fix an action leading to a winning history; specify arbitrary legal actions at its other histories. At every winning history where the adversary moves, all legal inputs preserve the winning status. These choices define a deterministic constructor strategy that wins against all adversary moves. The choices over arbitrary collections of histories may use the axiom of choice.
\end{proof}

\begin{remark}[Role of the information structure]
The determinization argument relies on perfect information: before choosing the next input, the adversary observes the constructor's realized actions, not merely their distributions. This differs from distribution-observing protocols in randomized online learning with set-valued or partial feedback~\cite{raman2024setvalued,shao2026partialfeedback}. Without access to the realized actions, the conclusion of Lemma~\ref{lem:finite-horizon-determinization} can fail, even in a two-round game. For example, suppose the constructor secretly chooses a fair bit and wins if it equals the adversary's second-round input. The constructor then wins with probability $1/2$ against every deterministic adversary, but no deterministic constructor strategy can guarantee a win.
\end{remark}

\begin{proof}[Proof of Proposition~\ref{prop:core}]
The case $d=0$ was proved in Section~\ref{subsec:foundations-separation}; a singleton $\mathcal H$ is covered exactly by its unique member. For $d\ge1$ and $b\le32\alpha d$, the deterministic family in \eqref{eq:accuracy-short-horizon} already proves the proposition. We therefore consider the remaining regime and use the randomized construction counted in Lemma~\ref{lem:expert-count}.

Let $\mathfrak S(\mathcal P)$ be the finite set of permutations of $\mathcal P$. Define a game in which the adversary chooses $x_t\in\mathcal X$ on round $t$, and the constructor then chooses
\begin{equation*}
  (\boldsymbol\xi_t,\boldsymbol\pi_t) \in\{0,1\}^k\times\mathfrak S(\mathcal P)^{A_*}.
\end{equation*}
Apply the capped schedule updates and routing protocol to these choices, processing the routing slots in their prescribed order, and emit the $K$ predictions before the next adversary move. All these operations are determined by the inputs and chosen actions. Designate a terminal history as winning precisely when \eqref{eq:determinization-realized-cover} holds.

Every deterministic adversary strategy in this game has the form \eqref{eq:routing-input-rule} with a constant input seed: its earlier inputs, and hence its current choice, can be recovered recursively from the preceding pool vectors and routing permutations. The randomized protocol therefore wins with probability greater than $1/2$ against every such strategy, by Lemma~\ref{lem:simultaneous-coverage}. Lemma~\ref{lem:finite-horizon-determinization} supplies a deterministic winning constructor strategy. Fix this strategy.

For any input prefix $x_{1:t}$, run the fixed strategy through round $t$. By induction on rounds, all pool vectors, routing permutations, assignments, and predictions through that round are determined by $x_{1:t}$. For each $\nu\in[K]$, define
\begin{equation*}
  g_{\nu,t}(x_{1:t}) :=\text{the output of index $\nu$ on round $t$ in this run}.
\end{equation*}
Thus $g_{\nu,t}:\mathcal X^t\to[0,1]$ is a well-defined deterministic causal map. The family
\begin{equation*}
  \mathcal G:=\{(g_{\nu,t})_{t=1}^T:\nu\in[K]\}
\end{equation*}
is fixed independently of the instance sequence. Since the strategy wins against every sequence of adversary inputs,
\begin{equation}
  \label{eq:determinization-global-outlier-cover}
  \forall x_{1:T}\in\mathcal X^T\;\forall h\in\mathcal H\; \exists\nu\in[K]:\qquad \sum_{t=1}^T \mathbf 1\{|g_{\nu,t}(x_{1:t})-h(x_t)|\ge8\eta\} <\frac{T}{8R}.
\end{equation}
Finally, $|\mathcal G|\le K\le(C_{\mathrm{cnt}}R)^d$. Choosing one universal $C_0\ge\max\{1,C_{\mathrm{cnt}},66e\alpha\}$ also includes the short-horizon bound \eqref{eq:accuracy-short-horizon}, and proves the proposition in all regimes.
\end{proof}

\begin{proof}[Proof of Theorem~\ref{thm:main}]
Fix $1\le p<\infty$, $T\ge1$, and $0<\varepsilon<1/2$, and put $\eta=\varepsilon/32$ and $d=\operatorname{sfat}_\eta(\mathcal F)$. If $d=0$, the midpoint rule
\begin{equation*}
  g_t(x_{1:t}) :=\frac12\left(\inf_{f\in\mathcal F}f(x_t) +\sup_{f\in\mathcal F}f(x_t)\right)
\end{equation*}
has pointwise error at most $\eta/2$ against every $f\in\mathcal F$, by Lemma~\ref{lem:rank-zero}. Its normalized $\ell_p$ error is therefore at most $\eta/2<\varepsilon$, and one expert suffices.

Suppose $d\ge1$. Lemma~\ref{lem:net} gives a separated subclass $\mathcal H\subseteq\mathcal F$ of rank at most $d$ such that
\begin{equation*}
  \forall f\in\mathcal F\;\exists h\in\mathcal H: \quad\|f-h\|_\infty\le\eta.
\end{equation*}
Apply Proposition~\ref{prop:core} with
\begin{equation}
  \label{eq:determinization-lp-parameters}
  R:=\frac{16}{\varepsilon^p}>32.
\end{equation}
The resulting global family has size at most $(16C_0/\varepsilon^p)^d$.

Fix an instance sequence $x_{1:T}$ and a target $f\in\mathcal F$. Choose $h\in\mathcal H$ with $\|f-h\|_\infty\le\eta$ and a witnessing expert $g$ for $h$ from Proposition~\ref{prop:core}. Writing
\begin{equation*}
  B:=\{t\in[T]:|g_t(x_{1:t})-h(x_t)|\ge8\eta\},
\end{equation*}
we have $|B|/T<1/(8R)$. For $t\notin B$,
\begin{equation*}
  |g_t(x_{1:t})-f(x_t)| \le |g_t(x_{1:t})-h(x_t)|+|h(x_t)-f(x_t)| <9\eta.
\end{equation*}
For $t\in B$, the error is at most one because both values lie in $[0,1]$. Consequently,
\begin{align}
  \label{eq:determinization-lp-error}
  \frac1T\sum_{t=1}^T|g_t(x_{1:t})-f(x_t)|^p &\le(9\eta)^p+\frac{|B|}{T}\notag\\
  &<\varepsilon^p\left[\left(\frac9{32}\right)^p +\frac1{128}\right]\notag\\
  &\le\frac{37}{128}\,\varepsilon^p <\varepsilon^p.
\end{align}
Here $(9/32)^p\le9/32$ because $p\ge1$. Taking the $p$th root gives the required normalized $\ell_p$ error bound. The separated subclass and the expert family were fixed independently of $f$ and $x_{1:T}$, so the same family satisfies Definition~\ref{def:dynamic-cover}. Taking $C=16C_0$ proves \eqref{eq:main-upper}, with $C$ independent of $p$, $d$, $T$, and $\varepsilon$.
\end{proof}

The upper-bound assertions of Corollaries~\ref{cor:binary} and~\ref{cor:tree} now follow from the rounding and tree-restriction arguments in Section~\ref{subsec:consequences}. Moreover, $C\ge1$ and $p\ge1$ give
\begin{equation*}
  \left(\frac{C}{\varepsilon^p}\right)^d \le\left(\frac{C}{\varepsilon}\right)^{pd},
\end{equation*}
which also proves the upper-bound part of Theorem~\ref{thm:intro-main}.

\section{Matching Lower Bounds}
\label{sec:lower}

We prove Theorem~\ref{thm:lower} by putting indicators of affine hyperplanes over fields of all power-of-two sizes into one class on a disjoint union of domains. For each parameter choice, the witnessing sequences are all sufficiently long prefixes of one infinite periodic sequence obtained by repeatedly listing every point of a selected finite-field block. Related finite-field hyperplane classes have been studied through their static VC dimension~\cite{ascoli2025hyperplanes}; here we require the Littlestone dimension of the countable union and a fixed-sequence $\ell_p$ packing.

\subsection{Finite-Field Hyperplanes and the Fixed Class}
\label{subsec:lower-construction}

Fix an integer $d\ge1$ and a prime power $q\ge2$, and let $\mathbb F_q$ be a field with $q$ elements. On the domain $\mathcal X_q:=\mathbb F_q^d$, define
\begin{equation}
  \label{eq:lower-hyperplanes}
  H_{a,b}:=\{x\in\mathbb F_q^d:a\cdot x=b\}, \qquad h_{a,b}(x):=\mathbf 1\{x\in H_{a,b}\},
\end{equation}
where $a\in\mathbb F_q^d\setminus\{0\}$, $b\in\mathbb F_q$, and $a\cdot x=\sum_{j=1}^d a_jx_j$ is computed in $\mathbb F_q$. Let
\begin{equation*}
  \mathcal H_q:= \{h_{a,b}:a\in\mathbb F_q^d\setminus\{0\},\ b\in\mathbb F_q\}.
\end{equation*}
Each $H_{a,b}$ is an affine subspace of dimension $d-1$ and has $q^{d-1}$ points: if $a_j\ne0$, choosing the other $d-1$ coordinates freely determines $x_j$ uniquely.

Two parameter pairs $(a,b)$ and $(a',b')$ define the same hyperplane exactly when $(a',b')=(\lambda a,\lambda b)$ for some $\lambda\in\mathbb F_q\setminus\{0\}$. Indeed, equal hyperplanes have the same direction subspace, so the nonzero linear forms $a\cdot x$ and $a'\cdot x$ have the same kernel and are proportional. Evaluating them at any point of the hyperplane then gives $b'=\lambda b$. The converse follows by multiplying the defining equation by $\lambda$. Thus each distinct hyperplane has exactly $q-1$ parameter representations, and
\begin{equation}
  \label{eq:lower-class-size}
  |\mathcal H_q| =\frac{q(q^d-1)}{q-1} =q\sum_{j=0}^{d-1}q^j \ge q^d.
\end{equation}

We now put all the finite-field classes on one domain. Define the disjoint union
\begin{equation}
  \label{eq:lower-fixed-domain}
  \mathcal X_d:=\bigsqcup_{m\ge1}\bigl(\{m\}\times\mathbb F_{2^m}^d\bigr).
\end{equation}
For $h_{a,b}\in\mathcal H_{2^m}$, define its extension by zero on the other blocks:
\begin{equation*}
  \widetilde h_{m,a,b}(n,x):=
  \begin{cases}
    h_{a,b}(x),&n=m,\\
    0,&n\ne m,
  \end{cases}
\end{equation*}
and take the union of all these extensions,
\begin{equation}
  \label{eq:lower-zero-extensions}
  \mathcal H_d:=\bigcup_{m\ge1}\{\widetilde h_{m,a,b}:h_{a,b}\in\mathcal H_{2^m}\}.
\end{equation}
Both $\mathcal X_d$ and $\mathcal H_d$ are countable and depend only on $d$.

\begin{lemma}[Dimension of the fixed class]
\label{lem:lower-fixed-class-dimension}
For every integer $d\ge1$ and $0<\gamma\le1$,
\begin{equation}
  \label{eq:lower-fixed-class-dimension}
  \operatorname{Ldim}(\mathcal H_d)=\operatorname{sfat}_{\gamma}(\mathcal H_d)=d.
\end{equation}
\end{lemma}
\begin{proof}
For the upper bound, use a predictor that initially predicts zero everywhere. After observing its first positive point in block $m$, it stores $m$ and the affine span $S$ of the positive points seen in that block. Thereafter it predicts one exactly for points in block $m$ that belong to $S$, enlarging $S$ only when it observes a new positive point in that block. On a sequence labeled by any $\widetilde h_{m,a,b}$, all positive points lie in its supporting block, and $S$ stays inside its target hyperplane. Hence the predictor makes only false-negative mistakes. Each such mistake raises the affine dimension of $S$ by one, starting at $-1$ before the first positive point and ending at most at $d-1$. Thus it makes at most $d$ mistakes.

If $\mathcal H_d$ shattered a Littlestone tree of depth $d+1$, we could follow the edge opposite the predictor's current prediction at each node. Some member of $\mathcal H_d$ realizes the resulting path, forcing $d+1$ mistakes, a contradiction. Therefore $\operatorname{Ldim}(\mathcal H_d)\le d$.

For the reverse inequality, work in the $m=1$ block and let $e_1,\ldots,e_d$ be the standard basis vectors in $\mathbb F_2^d$. We show that $\mathcal H_d$ realizes every binary labeling of $(1,e_1),\ldots,(1,e_d)$. Let $I\subseteq[d]$ be the indices to receive label one. If $I\ne[d]$, set $b=0$ and choose $a\in\mathbb F_2^d$ with $a_j=0$ for $j\in I$ and $a_j=1$ otherwise. Then $a\ne0$ and $h_{a,0}(e_j)=1$ exactly when $j\in I$. If $I=[d]$, take $a=(1,\ldots,1)$ and $b=1$, so $h_{a,1}(e_j)=1$ for every $j$. The corresponding zero-extended indicators belong to $\mathcal H_d$, so a complete binary tree that queries $(1,e_j)$ at every node of level $j$ is shattered. Hence $\operatorname{Ldim}(\mathcal H_d)\ge d$.

The binary-class identity from Section~\ref{subsec:sfat} now gives $\operatorname{sfat}_{\gamma}(\mathcal H_d)=d$ for every $0<\gamma\le1$.
\end{proof}

\subsection{Fixed-Sequence Packing and Lower Bound}
\label{subsec:lower-packing}

List all points of $\mathcal X_q$ once, in any fixed order, as $x_1,\ldots,x_L$, where $L=q^d$. For distinct hyperplanes $H,H'$ and their indicators $h,h'$, the normalized metric from \eqref{eq:normalized-lp} satisfies
\begin{equation}
  \label{eq:lower-distance-identity}
  \rho_{p,L}\bigl(h[x_{1:L}],h'[x_{1:L}]\bigr)^p =\frac{|H\mathbin{\triangle}H'|}{q^d} \qquad(1\le p<\infty),
\end{equation}
where $H\mathbin{\triangle}H'$ is the set of points belonging to exactly one of the two hyperplanes. This identity holds because the absolute difference of two binary values, raised to the power $p$, is one precisely when they differ.

If the normal vectors of $H$ and $H'$ are proportional, the two distinct hyperplanes are parallel and disjoint. Each has $q^{d-1}$ points, so the right-hand side of \eqref{eq:lower-distance-identity} is $2/q$. Otherwise the two normal vectors are linearly independent, which can occur only when $d\ge2$. Their two defining equations have $q^{d-2}$ common solutions, and hence
\begin{equation*}
  \frac{|H\mathbin{\triangle}H'|}{q^d} =\frac{2q^{d-1}-2q^{d-2}}{q^d} =\frac{2(q-1)}{q^2}.
\end{equation*}
For $d=1$, hyperplanes are singleton sets, so the value $2/q$ applies to every distinct pair. Thus, for all $d\ge1$,
\begin{equation}
  \label{eq:lower-trace-separation}
  \rho_{p,L}\bigl(h[x_{1:L}],h'[x_{1:L}]\bigr)^p \ge\frac{2(q-1)}{q^2} \qquad(h\ne h',\ h,h'\in\mathcal H_q).
\end{equation}

\paragraph{Choice of the field size.}
To obtain pairwise $\rho_{p,L}$-distance greater than $2\varepsilon$, fix $1\le p<\infty$ and $0<\varepsilon<2^{-1-1/p}$, and choose
\begin{equation}
  \label{eq:lower-field-choice}
  q:=2^{\left\lfloor\log_2\bigl((2\varepsilon)^{-p}\bigr) \right\rfloor}.
\end{equation}
Since $(2\varepsilon)^{-p}>2$, this is a power of two with $q\ge2$. It is the largest power of two not exceeding $(2\varepsilon)^{-p}$, so $q\le(2\varepsilon)^{-p}<2q$. Equivalently,
\begin{equation}
  \label{eq:lower-field-range}
  \frac{1}{2^{p+1}\varepsilon^p}<q \le\frac{1}{2^p\varepsilon^p}.
\end{equation}

The separation is strict in both possible cases:
\begin{equation*}
  \begin{aligned}
    q>2:&\qquad \frac{2(q-1)}{q^2}>\frac1q\ge(2\varepsilon)^p,\\
    q=2:&\qquad \frac{2(q-1)}{q^2}=\frac12>(2\varepsilon)^p.
  \end{aligned}
\end{equation*}
The first line uses \eqref{eq:lower-field-range}, and the second uses $\varepsilon<2^{-1-1/p}$. By \eqref{eq:lower-trace-separation}, we conclude that
\begin{equation}
  \label{eq:lower-lp-packing}
  \rho_{p,L}\bigl(h[x_{1:L}],h'[x_{1:L}]\bigr) >2\varepsilon \qquad(h\ne h',\ h,h'\in\mathcal H_q).
\end{equation}

\begin{proof}[Proof of Theorem~\ref{thm:lower}]
Take the single class $\mathcal H_d$ in \eqref{eq:lower-zero-extensions}. Lemma~\ref{lem:lower-fixed-class-dimension} gives all the asserted dimension identities. Fix $1\le p<\infty$ and $0<\varepsilon<2^{-1-1/p}$, and choose $q=2^m$ as in \eqref{eq:lower-field-choice}. With $L=q^d$ and $x_1,\ldots,x_L\in\mathcal X_q$ as enumerated above, define the infinite periodic instance sequence
\begin{equation}
  \label{eq:lower-periodic-sequence}
  \bar x_t:=\bigl(m,x_{1+((t-1)\bmod L)}\bigr)\in\mathcal X_d, \qquad t\ge1.
\end{equation}
Its period consists of the same block enumeration used in \eqref{eq:lower-lp-packing}.

The strict separation persists on every sufficiently long prefix. To make the horizon threshold explicit, set
\begin{equation}
  \label{eq:lower-horizon-threshold}
  \delta_q:=\frac{2(q-1)}{q^2}, \qquad
  T_0:=\left\lceil\frac{2L\delta_q}{\delta_q-(2\varepsilon)^p}\right\rceil.
\end{equation}
The strict inequality $\delta_q>(2\varepsilon)^p$ proved above makes $T_0$ finite. Fix any integer $T\ge T_0$. For distinct $h,h'\in\mathcal H_q$ and their zero extensions $\widetilde h,\widetilde h'$, each of the $\lfloor T/L\rfloor$ complete periods contributes the same discrepancy sum. Ignoring the nonnegative contribution of the final incomplete period and using \eqref{eq:lower-trace-separation} gives
\begin{equation}
  \label{eq:lower-periodic-packing}
  \begin{aligned}
  \rho_{p,T}\bigl(\widetilde h[\bar x_{1:T}],\widetilde h'[\bar x_{1:T}]\bigr)^p
  &\ge\frac{\lfloor T/L\rfloor L}{T}
  \rho_{p,L}\bigl(h[x_{1:L}],h'[x_{1:L}]\bigr)^p\\
  &\ge\left(1-\frac LT\right)\delta_q\\
  &\ge\frac{\delta_q+(2\varepsilon)^p}{2}>(2\varepsilon)^p.
  \end{aligned}
\end{equation}
The last line follows from $T\ge T_0$ and \eqref{eq:lower-horizon-threshold}.

An arbitrary real-valued prediction vector $w\in[0,1]^T$ cannot approximate two distinct zero-extended members $\widetilde h,\widetilde h'\in\mathcal H_d$ supported on this block to normalized $\ell_p$ error at most $\varepsilon$. Otherwise, the triangle inequality would give
\begin{equation*}
  \begin{aligned}
  \rho_{p,T}\bigl(\widetilde h[\bar x_{1:T}],\widetilde h'[\bar x_{1:T}]\bigr)
  &\le \rho_{p,T}\bigl(\widetilde h[\bar x_{1:T}],w\bigr)
  +\rho_{p,T}\bigl(w,\widetilde h'[\bar x_{1:T}]\bigr) \le2\varepsilon,
  \end{aligned}
\end{equation*}
contrary to \eqref{eq:lower-periodic-packing}. Thus any family of vectors covering all target vectors on this prefix has at least $|\mathcal H_q|$ members.

Every deterministic causal expert $g$ on $\mathcal X_d$ produces one such vector $g[\bar x_{1:T}]$ on this fixed prefix. Consequently, \eqref{eq:lower-class-size} and \eqref{eq:lower-field-range} yield
\begin{equation}
  \label{eq:lower-dynamic-bound}
  N_p^{\mathrm{dyn}}(\mathcal H_d,\varepsilon,T) \ge |\mathcal H_q| \ge q^d >\left(\frac{1}{2^{p+1}\varepsilon^p}\right)^d
  \qquad(T\ge T_0).
\end{equation}
The argument already allowed arbitrary vectors in $[0,1]^T$ and used the same periodic sequence for every $T\ge T_0$, proving the fixed-sequence assertion at all these horizons. The definition of $\mathcal H_d$ in \eqref{eq:lower-zero-extensions} is unchanged as $p$ and $\varepsilon$ vary.
\end{proof}

Combining \eqref{eq:lower-dynamic-bound} with Theorem~\ref{thm:main} and the definition \eqref{eq:worstcase} gives \eqref{eq:sharp-bounds} at every $T\ge T_0(d,p,\varepsilon)$. The comparisons in Sections~\ref{subsec:sharp-bounds} and~\ref{subsec:consequences} then complete Theorem~\ref{thm:intro-main} and the lower-bound assertions of Corollaries~\ref{cor:binary} and~\ref{cor:tree}.

\section*{AI Disclosure}
Several technical ideas underlying this work, particularly the detailed design of the routing mechanism and the tuning of its parameters, were developed through extensive discussions and iterative refinement involving the authors and ChatGPT 6 Pro. The authors also used ChatGPT 6 Pro to discuss the organization of the manuscript and to improve its clarity and presentation. The authors take full responsibility for all content of the paper, including the correctness of its mathematical results and proofs.

\printbibliography

\end{document}